\documentclass{article}

\usepackage{microtype}
\usepackage{graphicx}
\graphicspath{{figs/}}
\usepackage{subfigure}
\usepackage{amsthm}
\usepackage{amsmath}
\usepackage{nicefrac}
\usepackage{amssymb}
\usepackage{bbm}
\usepackage{booktabs} % for professional tables
\usepackage[table]{xcolor}
\usepackage{float}
\usepackage{pgf}
\usepackage{mathtools}
\usepackage{bbold}

\usepackage{thmtools}
\usepackage{thm-restate}
\usepackage{bibentry}
\usepackage{enumitem}
\usepackage{changepage} % needed for veryimportant custom env
\theoremstyle{definition}
\newtheorem{defi}{Definition}

\theoremstyle{plain}

\theoremstyle{remark}

\usepackage{multirow}
\definecolor{sky0}{HTML}{375E97}
\definecolor{sky1}{HTML}{5276AA}
\definecolor{sky2}{HTML}{1E4987}
\definecolor{sunset0}{HTML}{FB6542}
\definecolor{sunset1}{HTML}{FF8366}
\definecolor{sunset2}{HTML}{EF3C11}
\definecolor{sunflower0}{HTML}{FFBB00}
\definecolor{sunflower1}{HTML}{FFCA39}
\definecolor{sunflower2}{HTML}{C69100}

\definecolor{grass0}{HTML}{3F681C}
\definecolor{grass1}{HTML}{578134}
\definecolor{grass2}{HTML}{264809}

\definecolor{purple0}{HTML}{6a253c}
\definecolor{purple1}{HTML}{792a45}
\definecolor{purple2}{HTML}{88304d}
\definecolor{purple3}{HTML}{973556}
\definecolor{purple4}{HTML}{a14967}
\definecolor{purple5}{HTML}{ac5d78}
\definecolor{purple6}{HTML}{b67289}
\definecolor{purple7}{HTML}{c1869a}
\definecolor{purple8}{HTML}{cb9aab}
\definecolor{purple9}{HTML}{d5aebb}
\definecolor{purple10}{HTML}{e0c2cc}
\definecolor{purple11}{HTML}{ead7dd}
\definecolor{purple12}{HTML}{f5ebee}
\definecolor{green0}{HTML}{10423d}
\definecolor{green1}{HTML}{124b44}
\definecolor{green2}{HTML}{14534c}
\definecolor{green3}{HTML}{2c645e}
\definecolor{green4}{HTML}{437570}
\definecolor{green5}{HTML}{5b8782}
\definecolor{green6}{HTML}{729894}
\definecolor{green7}{HTML}{8aa9a6}
\definecolor{green8}{HTML}{a1bab7}
\definecolor{green9}{HTML}{b9cbc9}
\definecolor{green10}{HTML}{d0dddb}
\definecolor{green11}{HTML}{e8eeed}
\definecolor{gold0}{HTML}{9f7e35}
\definecolor{gold1}{HTML}{b6903d}
\definecolor{gold2}{HTML}{cca244}
\definecolor{gold3}{HTML}{e3b44c}
\definecolor{gold4}{HTML}{e6bc5e}
\definecolor{gold5}{HTML}{e9c370}
\definecolor{gold6}{HTML}{ebcb82}
\definecolor{gold7}{HTML}{eed294}
\definecolor{gold8}{HTML}{f1daa6}
\definecolor{gold9}{HTML}{f4e1b7}
\definecolor{gold10}{HTML}{f7e9c9}
\definecolor{gold11}{HTML}{f9f0db}

\usepackage{inconsolata}

\usepackage{xargs}  
\usepackage[colorinlistoftodos,prependcaption,textsize=tiny]{todonotes}
\newcommandx{\fix}[2][1=]{\todo[linecolor=red,backgroundcolor=red!25,bordercolor=red,#1]{#2}}
\newcommandx{\discuss}[2][1=]{\todo[linecolor=blue,backgroundcolor=blue!25,bordercolor=blue,#1]{#2}}

\definecolor{gred}{rgb}{0.86,0.27,0.22}
\definecolor{gblue}{rgb}{0.26,0.52,0.96}
\definecolor{ggreen}{rgb}{0.06,0.62,0.35}

\usepackage{hyperref}

\usepackage[accepted]{icml2020}

\icmltitlerunning{Topologically Densified Distributions}

\newcommand\drestr[2]{{% we make the whole thing an ordinary symbol
  \left.\kern-\nulldelimiterspace% automatically resize the bar with \right
  #1 % the function
  \vphantom{\big|} % pretend it's a little taller at normal size
  \right\|_{#2} % this is the delimiter
 }}

 \newcommand\indc{\mfc}
 
 \newcommand\metric{\mathbb{d}}

 \newcommand\featext{\varphi}
 \newcommand\cls{\gamma}

 \newcommand\parasep{\,;\,}

\newcommand{\wrt}{w.r.t.~}
\newcommand{\ie}{i.e.}
\newcommand{\eg}{e.g.}
\newcommand{\fp}{\mathfrak{p}}
\newcommand{\fq}{\mathfrak{q}}

\DeclareMathOperator{\bb1}{\mathbb{1}}
\DeclareMathOperator*{\expect}{\mathbb{E}}

\newcommand{\N}{\mathbb{N}}

\newcommand{\mcL}{\mathcal{L}}

\newcommand{\mcR}{\mathcal{R}}

\newcommand{\mcX}{\mathcal{X}}
\newcommand{\mcY}{\mathcal{Y}}
\newcommand{\mcZ}{\mathcal{Z}}

\newcommand{\mfc}{\mathfrak{c}}

\newcommand{\ttn}{\texttt{n}}

\newcommand{\ttB}{\texttt{B}}

\newcommand{\mbfz}{\mathbf{z}}

\newcommand\restr[2]{{% we make the whole thing an ordinary symbol
  \left.\kern-\nulldelimiterspace % automatically resize the bar with \right
  #1 % the function
  \vphantom{\big|} % pretend it's a little taller at normal size
  \right|_{#2} % this is the delimiter
  }}

\DeclareMathOperator{\supp}{supp}

\DeclareMathOperator*{\argmax}{\arg\!\max}

\pagecolor{white}
\begin{document}

\newcommand{\papertitle}{Topologically Densified Distributions}

\twocolumn[

  \icmltitle{\papertitle}



  \icmlsetsymbol{equal}{*}

  \begin{icmlauthorlist}
    \icmlauthor{Christoph D. Hofer}{sbg}
    \icmlauthor{Florian Graf}{sbg}
    \icmlauthor{Marc Niethammer}{unc}
    \icmlauthor{Roland Kwitt}{sbg}
  \end{icmlauthorlist}

  \icmlaffiliation{sbg}{Department of Computer Science, Univ. of Salzburg,
    Austria}
  \icmlaffiliation{unc}{Univ. of North Carolina, Chapel Hill, USA}
  \icmlcorrespondingauthor{Christoph D. Hofer}{\texttt{chr.dav.hofer@gmail.com}}

  \icmlkeywords{Machine Learning, ICML}

  \vskip 0.3in
]



\printAffiliationsAndNotice{}  % leave blank if no need to mention equal contribution

\nobibliography* % Needed for \bibentry to work 

\begin{abstract}
  We study regularization in the context of small sample-size
  learning with over-parameterized neural networks. Specifically,
  we shift focus from architectural properties, such as norms on the
  network weights, to properties of the
  internal representations before a linear
  classifier. Specifically, we impose a topological
  constraint on samples drawn from the probability
  measure induced in that space. This provably leads
  to mass concentration effects around the representations
  of training instances, \ie, a property beneficial for
  generalization. By leveraging previous work to impose
  topological constraints in a neural network setting,  
  we provide empirical evidence (across various vision benchmarks) 
  to support our claim for better generalization.
\end{abstract}

\section{Introduction}
\label{section:introduction}

Learning neural network predictors for complex tasks typically requires
large amounts of data. Although such models are over-parameterized, they
generalize well in practice. The mechanisms that govern generalization
in such settings are still only partially understood \cite{CZhang2017a}.
Existing generalization bounds \cite{Bartlett17a,Neyshabur17a,Golowich18,Arora18a} offer deeper
insights, yet the vacuity of the bounds and their surprising behavior in terms
of sample size \cite{Nagarajan19a} is a lasting concern.

\begin{figure}
  \centering{
    \includegraphics[width=\columnwidth]{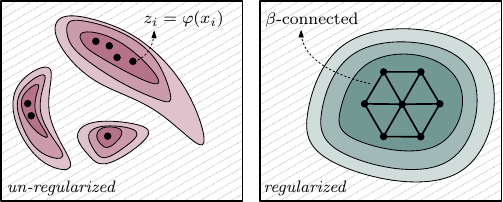}}
  \caption{Illustration of how topological regularization
  affects probability mass concentration in the internal 
  representation space of a neural network.  
  Darker shading denotes regions of higher probability mass. 
  \label{fig:title}}
  \vspace{-3pt}
\end{figure}

In \emph{small sample-size regimes}, achieving generalization is
considerably more challenging and, in general, requires
careful regularization, \eg, via various norms on the network
weights, controlling the Lipschitz constants, or adaptation 
%\discuss{weight controlling leads to lipschitz controlling?}
and adjustment of the training data. The latter does not
exert regularization on parts of the function
implemented by the network, but acts on
the training data, \eg, by regularizing its internal representations.
Prominent examples include modern augmentation techniques \cite{Cubuk19a}, 
or mixing strategies \cite{Verma19b} to control overconfident
predictions. Not only do these approaches show remarkable practical success, but, to some extent, can also be legitimized formally, \eg, through \emph{flattening}
%\discuss{Is there a formal justification why flat represenations are better?} 
arguments in the representation space, or through variance reduction 
arguments, as in case of data augmentation \cite{Dao19a}.

In this work, we contribute a regularization approach that
hinges on \emph{internal representations}. 
We consider neural networks as a functional composition of the form
\begin{equation}
\cls  \circ \featext:
  \mathcal{X}\to \mathcal{Y} = \{1,\ldots,K\} \enspace,
  \label{eqn:network}
\end{equation}
where $\featext: \mathcal{X} \to \mathcal{Z}$ denotes a high-capacity
feature extractor which maps into an internal representation space 
$\mcZ$. A linear classifier $\cls: \mathcal{Z} \to 
\mathcal{Y}$ then predicts one of $K$ classes. 
As customary, $\gamma$ is typically of the form $Az+b$, followed
by the $\argmax$ operator. In our setting,  
we focus on the representation space $\mathcal{Z}$ and, 
specifically, on the \emph{push-forward
  probability measure} induced by $\featext$ on $\mathcal{Z}$. 
We then identify a \emph{property of this measure
that is beneficial to generalization}. This shifts attention away from
properties of the network and instead focuses on
its \emph{utility} to implement these properties
during learning, \eg, by means of regularization.
  
%This is different to the information bottleneck, which aims to control information
%flow through $\gamma \circ \varphi$ and comes at the
%price of notable technical challenges \citep[see][]{Goldfeld19a,Kolchinsky19a}.

Our \textbf{contributions} are as follows: First, we formalize the intuition 
that in a neural network regime (where training samples are fitted perfectly), 
generalization is linked to probability mass concentration around the training 
samples in the internal representation space $\mcZ$. Second, we prove that a 
\emph{topological constraint} on samples from the aforementioned push-forward measure 
(restricted to each class), leads to mass concentration.
Third, relying on work by \citet{Hofer19a}, we  
devise a regularizer to encourage the derived topological
constraint during optimization and present empirical evidence across
various vision benchmark datasets to support our claims.

\begin{figure*}[t!]
  \centering{
    \includegraphics[width=\textwidth]{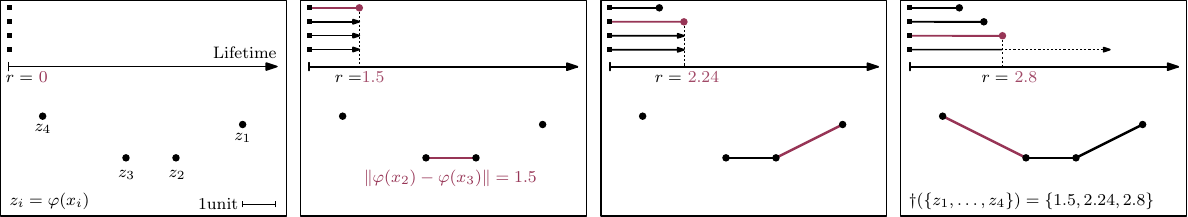}}
  \caption{Illustration of \emph{0-dimensional Vietoris-Rips persistent homology}.  
  Starting from a point set $M = \{z_1,\ldots, z_4\}$, we iteratively construct a graph by sweeping $r \in [0,\infty)$.
  While increasing $r$, we add an edge between $z_i \neq z_j$ if $\|z_i-z_j\|_2=r$ and track 
  the \emph{merging} of connected components as we successively add edges. The times (aka \emph{death-times}) of those 
  \emph{merging events} are stored in the (multi-)set $\dagger(M)$, typically called
  a \emph{persistence barcode} (as all connected components appear at $r=0$, we omit the \emph{birth} times). \label{fig:vr}}
\end{figure*}
\subsection{Related work}
\label{subsection:relatedwork}

Prior works, most related to ours, focus on regularizing 
statistics of internal representations 
in supervised learning settings. As opposed to \emph{explicit}
regularization which aims at restricting model capacity, \eg, by  penalizing norms on network weights, 
regularizing representation statistics can be considered a less direct, 
data dependent, mechanism. 

The intended objective of regularizing representation statistics 
vary across the literature. \citet{Glorot11a}, for instance, encourage
sparsity via an $l_1$ norm penalty.
\citet{Cogswell16a} aim for redundancy
minimization by penalizing the covariance matrix of representations. \citet{Choi19a} 
later extended this idea to perform 
class-wise regularization, \ie, a concept similar to \cite{Belharbi17a} 
where pairwise distances between representations of a class are 
minimized. In an effort to potentially replace batch normalization, 
\citet{Littwin18a} propose to control variations, 
across mini-batches, of each neuron's variance (before activation). 
This is shown to be beneficial in terms of reducing the number of modes 
in the output distribution of neurons. \citet{Liao16a} follow a different 
strategy and instead cluster representations to achieve parsimony, but 
at the cost of having to set the number of clusters (which can be 
difficult in small-sample size regimes). Motivated by the 
relation between the generalization error and the natural gradient \cite{Roux07a}, 
\citet{Joo20a} recently proposed to match the distribution of  
representations to a standard Gaussian, via the sliced Wasserstein distance.  
Yet, to the best of our knowledge, all mentioned regularization 
approaches in the realm of controlling internal representations,
either only empirically demonstrate better generalization, or show
a loose connection to the latter. 
\emph{In contrast, we establish a direct (provable) connection between a property 
encouraged by our regularizer and the associated beneficial effects on 
generalization.}

Our technical contribution resides on the 
intersection of machine learning and algebraic topology 
(persistent homology in particular). 
Driven by various intents, several works have recently adopted concepts from 
algebraic topology.
\citet{Rieck19a}, for instance, study topological aspects
of neural networks, represented as graphs, to guide architecture 
selection, \citet{Guss18a} aim to quantify dataset complexity.
On the more theoretical side, \citet{Bianchini14a} study
the functional complexity of neural networks. Notably, 
these works \emph{passively use} ideas from topology 
for post-hoc analysis of neural networks. More recently, various works 
have presented progress along the lines of \emph{actively} controlling  topological aspects.
\citet{Chen19a} regularize decision boundaries of
classifiers, \citet{Hofer19a} optimize connectivity of internal
representations of autoencoders, and \citet{Rieck19a} match
topological characteristics of input data to the topological 
characteristics of representations learned by an autoencoder. 
Technically, we rely on these advances to eventually  
implement a regularizer, \emph{yet our primary objective is to
study the connection between generalization and the 
topological properties of the probability measure induced
by a neural network's feature extractor $\varphi$.}

\subsection{Notation \& Learning setup}

In the context of Eq.~\eqref{eqn:network}, we refer to
$\mathcal{X}$, $\mathcal{Y}$, $\mathcal{Z}$ as the sample,
label and internal representation space. We assume that 
$\mcZ$ is equipped with a metric
$\metric$ and $\mcY =  [K] = \{1,\ldots,K\}$. By $P$, we denote a probability measure on $\mathcal{X}$
and by $Q$ the push-forward measure, induced by the measurable
function $\featext: \mathcal{X} \to \mathcal{Z}$, on the Borel
$\sigma$-algebra $\Sigma$ defined by $\metric$ on $\mcZ$; 
$Q^b$ denotes the product measure of $Q$.   
%on the Cartesian product of $b$ copies of $\mathcal{Z}$. 
$B(z,r)$ refers to the \emph{closed} (metric) ball of radius 
$r>0$ around $z \in \mathcal{Z}$.

Our 
learning setup is to assume a deterministic
relationship\footnote{\ie, an arguably realistic setup \cite{Kolchinsky19a} for many practical problems.} between $y \in \mathcal{Y}$ and $x \in \mathcal{X}$. This
relationship is determined by $c: \supp(P) \to \mathcal{Y}$, where
$\supp(P)$ refers to the support of $P$.
We assume a training sample $S$, consisting of pairs
$(x_1,y_1),\ldots,(x_m,y_m)$ is the result of $m$ i.i.d. draws of $X \sim P$,
labeled via $y_i = c(x_i)$. By $S_{x|k}$ we
denote the data instances, $x_i$, of class $k$.
For a classifier 
$h: \mcX \rightarrow \mcY$ and $X \sim P$, we define
the \emph{generalization error} as
\[
  \expect\limits_{X \sim P}[\mathbb{1}_{h, c}(X)]\enspace,
\]
where
\[
  \mathbb{1}_{h, c}(x) =
  \begin{cases}
    0, \quad h(x) = c(x), \\
    1, \quad \text{else}\enspace.
  \end{cases}
\]

For brevity, proofs are deferred to the suppl. material.

\section{Topologically densified distributions}
\label{section:topologicallyregularizeddistributions}

To build up intuition, consider
$X \sim P$ and $\featext(X) = Z$. As $Z \sim Q$ and the
linear classifier $\cls$ operates on the internal
representations yielded by $\featext$, we can ask
two questions: \emph{(I)
Which properties of $Q$ are beneficial for
  generalization, and (II) how can we impose these properties?}

Increasing the probability that $\varphi$ maps a sample of class $k$ 
into the correct decision region (induced by $\cls$) improves 
generalization. In Lemma~\ref{lem:generalization-in-feature-space} 
we will link this fact to a condition depending on $Q$. 

At first, we introduce a way to measure \emph{class-specific 
probability mass}. 
To this end, we define the restriction of $Q$ (\ie, the push-forward
of $P$ via $\varphi$) to class $k$ 
as
\begin{equation}
  Q_k:\Sigma \rightarrow [0, 1]
  \quad
  \Sigma \ni \sigma \mapsto \frac{Q(\sigma \cap C_k)}{Q(C_k)}\enspace,
  \label{eqn:Q_k}
\end{equation}
where $C_k = \varphi(c^{-1}(\{k\}))$ is the
representation of class $k$ in $\mcZ$.
In the optimal case, the probability mass of the $k$-th decision region, 
measured via $Q_k$, tends towards one. The following 
lemma formalizes this notion by establishing a direct link 
between $Q_k$ and the generalization error.

\vskip1ex
\begin{restatable}{lem}{lem@generalization@in@feature@space}
  \label{lem:generalization-in-feature-space}
  For any class $k \in [K]$,  let $C_k = \featext\left(c^{-1}\big(\{k\}\big)\right)$ be its internal representation 
  and $D_k = \cls^{-1}\big(\{k\}\big)$ be 
  its decision region in $\mcZ$ \wrt $\gamma$.
  If, for $\varepsilon>0$\ , 
  \begin{equation}
  \forall k: 1 - Q_k(D_k) \leq \varepsilon\enspace,
  \label{eqn:lemma1_cond}
  \end{equation}
  then
  \[
    \expect\limits_{X \sim P}[\mathbb{1}_{\cls\circ\featext, c}(X)]
    \leq K\varepsilon
    \enspace.
  \]
\end{restatable}

While Lemma~\ref{lem:generalization-in-feature-space} partially answers \emph{Question (I)}  
by yielding a property beneficial for generalization, it remains to find a mechanism
to impose it.

%Lemma~\ref{lem:generalization-in-feature-space} essentially states that the 
%probability mass difference, measured in $\mcZ$, of class $k$ and the 
%its intersection with the corresponding decision region of the classifier, 
%controls the generalization error. A schematic illustration Eq.~\eqref{eqn:lemma1_cond} is shown in Fig.~\ref{fig:lemma1} for a two-class setting.

\begin{figure}[H]
\centering{
\includegraphics[width=\columnwidth]{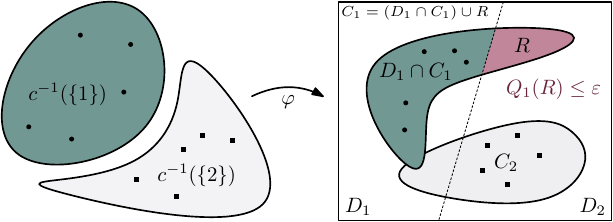}}
\caption{\label{fig:lemma1} Eq.~\eqref{eqn:lemma1_cond} of Lemma~\ref{lem:generalization-in-feature-space} controls how much 
probability mass of $C_k$ is concentrated in $D_k$ \wrt $Q_k$ (only illustrated for $k=1$ here). The smaller $\varepsilon$ gets, 
the less mass is present in $R$, \ie, the region 
where errors on unseen data (of class $k=1$) occur.}
\end{figure}

Inspired by recent work of \citet{CZhang2017a}, we continue by assuming 
$\cls~\circ~\featext$ to be powerful enough to fit any given training set $S$ 
of size $m$. Specifically, we assume that for $z_i = \featext(x_i)$ with $c(x_i) = k$, 
there exists $r>0$ with $B(z_i,r) \subset D_k$.
%\footnote{That is, we exclude the practically irrelevant case of having training samples arbitrarily close to the decision boundary in $\mcZ$.}
This is equivalent to a margin assumption on the training instances in $\mathcal{Z}$.
With this in mind, increasing $Q_k(B(z_i,r))$ is beneficial for generalization, as
it can only increase $Q_k(D_k)$. Our strategy hinges on this idea. 

\subsection{Topological densification}
\label{subsection:topodense}

We show that a certain topological constraint on
$Q_k$ will \emph{provably} lead to \emph{probability 
mass concentration}. More precisely, given a reference 
set $M \subseteq \mcZ$ and its $\varepsilon$-extension
\begin{equation}
  M_{\varepsilon} = \bigcup\limits_{z \in M} B(z, \varepsilon), \quad \varepsilon>0\enspace,
\label{eqn:extension}
\end{equation}
the topological constraint provides a non-trivial lower bound on 
$Q_k(M_{\varepsilon})$ in terms of $Q_k(M)$.
Informally, we say that $Q_k$ is \emph{topologically densified} around $M$. 

Our main arguments rely on the probability of an i.i.d. draw from
$Q_k^b$ (\ie, the product measure) to be \emph{connected}. The latter is a topological property
which can be computed using tools from algebraic topology. In
particular, we quantify connectivity via \emph{$0$-dimensional (Vietoris
Rips) persistent homology}, visually illustrated in
Fig.~\ref{fig:vr}. 
%This illustration should be sufficient to understand how persistent homology 
%captures connectivity. 
For a thorough technical introduction, we refer the reader to, \eg, \cite{Edelsbrunner2010} 
or \cite{Boissonnat18a}.

\vskip1ex
\begin{defi}
  Let $\beta > 0$. A set $M \subseteq \mcZ$ is \emph{$\beta$-connected} iff all $0$-dimensional death-times of its Vietoris-Rips persistent homology are in the \emph{open} interval $(0, \beta)$.
  \label{defi:beta-connected}
\end{defi}

As all information captured by $0$-dimensional Vietoris-Rips persistent homology \cite{Robins00a} is 
encoded in the minimum spanning tree (MST) on $M$ (\wrt metric $\metric$), we can 
equivalently formulate Definition~\ref{defi:beta-connected} in terms of the edges in the MST. In particular, 
we can say that each edge in the MST of $M$ has edge length less than $\beta$. However, the topological 
perspective is preferable, as we can rely on previous work \cite{Hofer19a} which shows how to backpropagate gradients
through the persistent homology computation. This property is needed to implement a regularizer (see \S\ref{subsection:regularization}).

To capture $\beta$-connectedness of a sample, we define the indicator function
$\indc_b^{\beta}: \mathcal{Z}^b \to \{0,1\}$ as
\[
  \indc_b^{\beta}(z_1, \dots, z_b) = 1 \Leftrightarrow \{z_1, \dots, z_b\}~\text{ is }~\beta\text{-connected}
  \enspace.
\]
%If it is clear from context, we omit sub- and superscripts and write
%If the indicator function
%is applied to $b$ random variables, $Z_1, \dots, Z_b$, we will abuse
%notation and write $\indc$ instead of $\indc(Z_1, \dots, Z_b)$.

We now consider the probability of $b$-sized i.i.d. draws
from $Q_k$, see Eq.~\eqref{eqn:Q_k}, to be $\beta$-connected. 
%This is 
%necessary, as requiring $\beta$-connectedness of \emph{all} $b$-sized
%would imply that the dimensionality of $\mathcal{Z}$ has to increase
%as $b$ increases (by a simple metric entropy argument).

\vskip1ex
\begin{defi}
  \label{defi:bbetac}
  Let $\beta>0, c_{\beta} \in [0,1]$, and $b\in \N$. 
  We call $Q_k$ \emph{$(b,c_{\beta})$-connected} if 
  %If $Q_k$ has $c_{\beta}$-probable $\beta$-connected $b$-samples, \ie, 
  \[
    Q_k^b(\{\indc_b^{\beta} = 1 \}) \geq c_{\beta}\enspace.
    %\enspace,
  \]
  %we call it \emph{metrically gathered (\wrt radius $\beta$, power $b$ and probability $c_{\beta}$)}.
\end{defi}

To underpin the relevance of this (probabilistic) connectivity property 
\wrt probability mass concentration, we sketch the key insights that lead to the main results of the paper.

{\bf Connectedness yields mass concentration.}
%We argue that the formulation of the theoretical key insights of this work is somewhat unhandy. 
For sake of argument, assume $Q_k$ to be $(b,c_{\beta})$-connected. 
Now, consider a reference set $M \subseteq \mcZ$ together with two sets   
\[
  N = M_{\beta}^{\complement} 
  \quad
  \text{ and } 
  \quad
  O = M_{\beta}\setminus M\enspace,
\]
where $M_{\beta}^{\complement}$ denotes the set complement. These 
three sets are illustrated below on a toy example.
\begin{center}
  \label{fig:M}
\includegraphics[width=0.99\columnwidth]{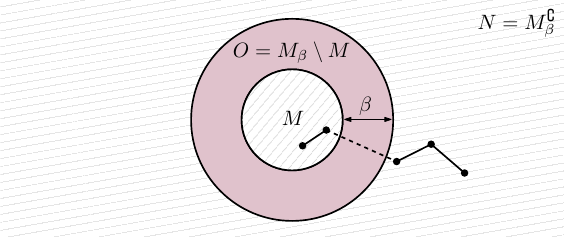}
\end{center}

Apparently, $M, N$ and $O$ partition $\mcZ$.
For an i.i.d. sample, $z_1, \dots, z_b$, from $Q_k$, we can study how 
distributing the $z_i$ among $M$, $N$, and $O$ impacts $\beta$-connectedness. 
In particular, let $\sharp_M$, $\sharp_N$, $\sharp_O$ be the number of $z_i$'s which fall within $M$, $N$, and $O$ respectively, \ie,
\[
  \sharp_M = |\{z_i\} \cap M|,\,\,
  \sharp_N = |\{z_i\} \cap N|,\,\,
  \sharp_O = |\{z_i\} \cap O|\enspace. 
\]

\underline{Observation 1}: If the membership assignment is such that 
\[
   \sharp_M \geq 1
   \text{ and }
   \sharp_N \geq 1
   \text{, but }
   \sharp_O = 0\enspace, 
\]
then $z_1, \dots, z_b$ cannot be $\beta$-connected. This is easy to see, 
as for $z_M \in M$ and $z_N \in N$, we get $\metric(z_M, z_N) \geq \beta$ and there are simply no $z_i$'s in $O$ (see illustration above).

\underline{Observation 2}: The probability of 
\[
  (\sharp_M, \sharp_N, \sharp_O)  \in \{0, \dots, b\}^3
\]
is given by a Trinomial distribution with parameters $Q_k(M)$, $Q_k(O)$ and $Q_k(N) = 1 - Q_k(M) - Q_k(O)$. 

As it holds that  $Q_k(O) = Q_k(M_\beta) - Q_k(M)$ and $Q_k(N) = 1 - Q_k(M_{\beta})$, the probability 
of a $b$-sample \emph{not} being $\beta$-connected can be expressed in terms of 
\[
  \fp = Q_k(M) \quad\text{and}\quad \fq = Q_k(M_{\beta})
  \enspace. 
\]
The key aspect of this construction is that we describe 
\emph{events} where $z_1, \dots, z_b$ cannot be $\beta$-connected, \ie, 
\[
  E = \{(z_i) \in \mcZ^b: \sharp_M \geq 1,  \sharp_N \geq 1, \sharp_O = 0\}
\]
and $\indc_b^{\beta}(E)=\{0\}$.
As we will see, based on the Trinomial distribution, one can derive a polynomial $\Psi$ expressing 
the probability of $E$, \ie, 
\[
  Q_k(E) = \Psi(\fp, \fq)
  \enspace.
\]

Consequently, as $c_{\beta}$ is defined to be the probability 
of a $b$-sample to be $\beta$-connected and $E$ describes events where 
$b$-samples are \emph{not} $\beta$-connected, it holds that 
\[
  1 - c_{\beta} \geq Q_k(E) = \Psi(\fp, \fq)
  \enspace.
\]
We will see, by properties of $\Psi$, that this relationship allows us to lower bound $\fq = Q_k(M_{\beta})$, 
if $\fp = Q_k(M)$ is known. In other words, if $M$ covers a certain mass, then we can infer the minimal mass  
which has to be covered by $M_{\beta}$.  
Our main result -- presented in Theorem~\ref{cor:beta-margin-concentration} -- is slightly more general, as it 
not only considers the $\beta$-extension of $M$, but the $l\cdot\beta$-extension for $l\in\N$.
In this more general case, the polynomial $\Psi$ takes the form as in Definition~\ref{def:psi-function}
below.

\vskip1ex
\begin{restatable}{defi}{restatable@defi@psifunction}
  \label{def:psi-function}
  Let $b, l \in \N$ and $p, q \in [0, 1]$. For $p \leq q$, 
  we define the polynomial
  \begin{align*}
    \Psi(p, q \parasep b, l) & =
    \sum\limits_{\substack{(u,v,w) \\ \in I(b, l)}}
    \frac{b!}{u!\ v!\ w!}
    p^{u}
    (1-q)^{v}
    (q-p)^{w}\enspace,
  \end{align*}
  where the index set $I(b,l)$ is given by
  \begin{align*}
    I(b, l) =
    \left\{\right.
     & (u,v,w) \in \mathbb{N}_0^3:                     \\
     & u+v+w = b \;\wedge\; u,v \geq 1 \;\wedge\; w \leq l - 1
    \left.\right\}
    \enspace.
  \end{align*}
\end{restatable}

The most important properties of $\Psi$ are: (1) $\Psi$  is \emph{monotonically increasing} in $p$ (and $l$); 
(2) $\Psi$ is \emph{monotonically decreasing} in $q$ and (3) $\Psi$ vanishes for $q=1$.

\vskip1ex
\begin{restatable}{thm}{restatable@thm@psifunction}
  \label{cor:beta-margin-concentration}
  Let $b,l\in \N$ and let $Q_k$ be $(b,c_{\beta})$-connected. Then, for all reference sets 
  $M \in \Sigma$ 
  and
  \[
    \fp=Q_k(M),\ \fq=Q_k(M_{l\cdot\beta})
  \]
  it holds that
  \begin{equation}
    \label{eq:beta-margin-concentration}
    1-c_{\beta}
    \geq
    \Psi\big(\fp, \fq \parasep b, l \big)\enspace.
  \end{equation}
\end{restatable}

\subsection{Ramifications of Theorem~\ref{cor:beta-margin-concentration}}

By properties (1) -- (3) of $\Psi$, 
Theorem~\ref{cor:beta-margin-concentration}
allows us to lower-bound the mass increase caused by extending (see Eq.~\eqref{eqn:extension}) 
a reference set $M$ by $l\cdot\beta$. Recall that this is beneficial for  
generalization, if $M$ is constructed from representations (in $\mcZ$) of correctly classified 
training instances.

In detail, assume that the mass of the reference set $M$, $\fp = Q_k(M)$, is fixed 
and let $\fq = Q_k(M_{l\cdot\beta})$ be the mass of the $l\cdot\beta$ extension. 
Then, by Theorem~\ref{cor:beta-margin-concentration}, 
\begin{equation}
\fq \in \Big\{
  q \in [\fp,1]:
  1 - c_{\beta}
  \geq
  \Psi(\fp, q \parasep b, l)
  \Big\} = A\enspace,
\label{eqn:setA}
\end{equation}
and thus $A$ is non-empty. Now let $\mcR_{b, c_{\beta}}(\fp, l) = \min A$ 
identify the \emph{smallest mass} in the $l\cdot\beta$-extension 
for which the inequality in Eq.~\eqref{eq:beta-margin-concentration}
holds. As $\Psi$ is monotonically decreasing, $\mcR_{b, c_{\beta}}(\fp, l)$ 
is monotonically increasing in $c_{\beta}$. This will motivate our 
regularization goal of \emph{increasing $c_{\beta}$}.

\textbf{Sufficient condition for mass concentration.} Note that 
$\fq \geq \mcR_{b, c_{\beta}}(\fp, l) \geq \fp$ and thus,  
mass concentration is guaranteed as long as $\mcR_{b, c_{\beta}}(\fp, l) > \fp$. Otherwise, the mass in the $l\cdot\beta$-extension of $M$ may not be greater than the mass in $M$. In fact,  $\mcR_{b, c_{\beta}}(\fp, l) > \fp$ only holds if
%
%
%Then, if $1 - c_{\beta} \not\geq \Psi(Q_k(M), Q_k(M) \parasep b, l)$
%
%
%be fixed and $q_0$ such that
%$1 - c_{\beta} \not\geq \Psi(p_0,q_0 \parasep b, l)$.
%Then, as $\Psi$ is monotonically decreasing in $q$ \emph{and} 
%$1 - c_{\beta} \geq \Psi(p, q \parasep b, l)$ 
%needs to be satisfied,  there exists a minimal value $q \geq p$ fulfilling this requirement, \ie, 
%
%As $Q_k(M_{l\cdot\beta}) \geq \mcR_{b, \beta}(p, l)$ and thus
%\[
%Q_k(M_{l\cdot\beta}) - Q_k(M) \geq \mcR_{b, \beta}(p, l)-p
%\]
%
%
%, this allows us to \emph{lower-bound} the mass increase when extending $M$ by a margin of $l\cdot\beta$, as $Q_k(M_{l\cdot\beta}) - Q_k(M) \geq \mcR_{b, \beta}(p, l)-p$.
%Thus, it is beneficial that $\mcR_{b, \beta}(p, l) \gg p$. 
\begin{equation}
1-c_{\beta} < \Psi(\fp,\fp \parasep b, l) = 1-\fp^b-(1-\fp)^b
\enspace.
\label{eqn:criticalmass}
\end{equation}

The behavior of Eq.~\eqref{eqn:criticalmass} is specifically
relevant in the region where $\fp$ is close to $0$, as requiring
a large mass in $M$ would be detrimental. Notably, as we can see in Fig.~\ref{fig:criticalmass}, which shows Eq.~\eqref{eqn:criticalmass} as 
a function of $\fp=Q_k(M)$, already small values of $\fp$ reach the 
\emph{critical} threshold of $1-c_{\beta}$.
\vskip1ex
\begin{figure}[H]
\centering{
\includegraphics[width=0.99\columnwidth]{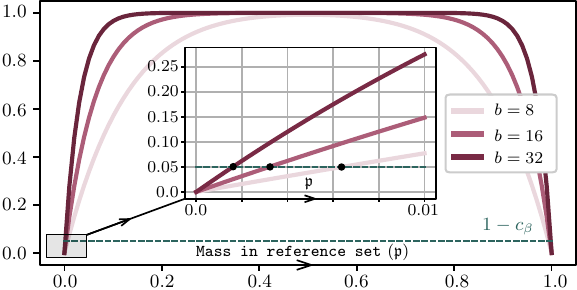}}
\vspace{-3pt}
\caption{Illustration of when $1-\fp^b-(1-\fp)^b > 1-c_{\beta}$ holds, 
\ie, when mass concentration effects start to occur. 
The zoomed-in view shows
the relevant region near $0$.\label{fig:criticalmass}}
\end{figure}

\textbf{Quantification of mass concentration}. To understand how the minimal mass in $M_{l \cdot \beta}$ is boosted by 
the mass of the reference set $M$, we visualize (in Fig.~\ref{fig:bound2}) 
%the set $A$ from Eq.~\eqref{eqn:setA}, \ie, 
the minimial values for 
$\fq = M_{l \cdot \beta}$
%, as well as its minimum $ \mcR_{b, c_{\beta}}(\fp, l)$,
as a function of 
%the mass 
$\fp=Q_k(M)$, \ie, the mass of the reference set $M$.
Similar to Fig.~\ref{fig:criticalmass}, as $\fp$ approaches $0$ (or $1$), 
the mass concentration effect is rendered negligible.
However, already a small mass in $M$ is sufficient for strong
mass concentration in $M_{l\cdot\beta}$. 
\begin{figure}[t!]
\centering{
\includegraphics[width=0.99\columnwidth]{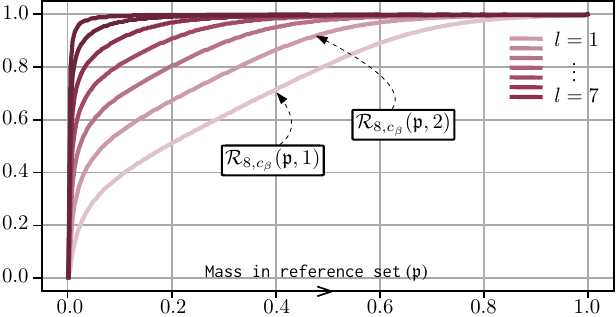}}
\vspace{-3pt}
\caption{Illustration of $\mcR_{b, c_{\beta}}(\fp, l)$, \ie, 
the lower bound on $\fq = Q_k(M_{l\cdot\beta})$,  
plotted as a function of the mass $\fp=Q_k(M)$ of the reference set $M$ (for $b=8$ and 
different $l$).
\label{fig:bound2}}
\end{figure}

Next, we discuss the role of $c_{\beta}$, \ie, the probability of a $b$-sized sample from $Q_k$ to be $\beta$-connected.
Fig.~\ref{fig:bound1} illustrates, for different choices of $l$, where $1-c_{\beta} \geq \Psi(\fp,\fq\parasep b,l)$ holds, as a function
of $\fq$ with $\fp=0.1$ fixed. 
Most importantly, as $c_{\beta}$ is increased, the 
minimal mass in a particular $l\cdot\beta$ extension of $M$, characterized by $\mcR_{b, c_{\beta}}(\fp, l)$,
shifts towards larger values. 

\begin{figure}[H]
\centering{
\includegraphics[width=0.99\columnwidth]{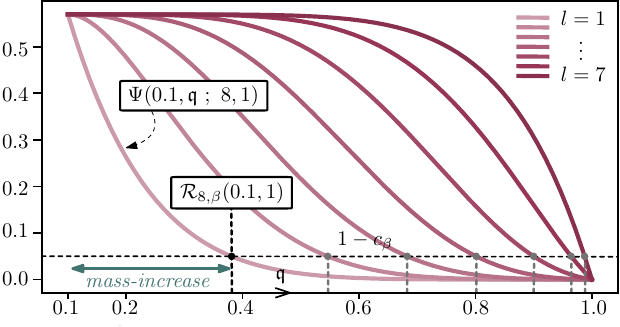}}
\vspace{-3pt}
\caption{\label{fig:bound1} Illustration of $\Psi$ 
for $\fp=0.1, b=8$ and different $l \in \mathbb{N}$. Points at which
$1-c_{\beta} = \Psi(\fp,q~;~b,l)$ holds are marked 
by dots.}
\end{figure}

\emph{Overall}, Theorem~\ref{cor:beta-margin-concentration} and the 
analysis presented above provide a possible answer to 
$\emph{Question (II)}$ stated at the beginning of \S\ref{section:topologicallyregularizeddistributions}. 
In particular, a mechanism to increase $Q_k(D_k)$ in
Eq.~\eqref{eqn:lemma1_cond} is to encourage $Q_k$ to 
be $(b, c_{\beta})$-connected. We conclude with the 
following summary:

\vspace{-3pt}
\begin{center}
\emph{If a measure is $(b, c_{\beta})$-connected, then 
mass attracts mass; the higher  
$c_{\beta}$, the stronger the effect.}
\end{center}
%
%
%\textbf{Implications for generalization.}
%As mentioned earlier, it is typical for today's 
%neural networks to perfectly fit the training 
%set $S$. Thus, for a given class $k$, 
%mass concentration around the internal representations 
%of training instances is beneficial. In particular, if $Q_k$ is $(b,c_{\beta})$-connected and there is a neighborhood, $M$,  
%around $z_i = \varphi(x_i)$ for $c(x_i)=k$, with $p = Q_k(M)$, 
%then we have shown that mass will be accumulated in $M_{l\cdot\beta}$. 
%In fact, we can characterize this effect in terms of $c_{\beta}$, $l$, and $p$. 

\subsection{Limitations}
\label{subsection:limitations}

As the mass, $\fp$, of the reference set $M$ is typically 
unknown, all results are \emph{relative} (\wrt $\fp$), not absolute. 
This warrants a discussion of potential limitations in a 
learning context.

Specifically, when learning from samples, an arguably 
natural choice for a class-specific 
reference set is to consider the union of balls
around the representations of the training samples, yielding (for $r>0$)
\begin{equation}
  M^{(k)} = \bigcup_{\mathclap{z \in \varphi(S_{x|k})}}\ B(z, r)\enspace.
  \label{eqn:trainingreferenceset}
\end{equation}
Now, lets assume that $Q_k$ is $(b, c_{\beta})$-connected
and the linear classifier $\cls$ attains zero error on $\varphi(S_{x|k})$.
Two issues suppress good generalization: 

\emph{First}, the derived mass concentration is only beneficial if,  
for a given $l \in \N$, the reference set $M^{(k)}$ is located 
sufficiently far away from the decision boundary of
class $k$, \ie, 
\begin{equation}
M_{l\cdot\beta}^{(k)} \subset D_k = \gamma^{-1}\big(\{k\}\big)\enspace.
\label{eqn:subsetcondition}
\end{equation}
In practice, we can, to some extent, induce such a configuration by
selecting a loss function which yields a large margin in $\mcZ$. 
In that case, at least a $\mcR_{b,c_{\beta}}(Q_k(M^{(k)}),l)$ 
proportion of class $k$ is correctly classified by $\gamma$. 
A violation of Eq.~\eqref{eqn:subsetcondition} would mean that 
mass is still concentrated, but the $l\cdot\beta$-extension might 
reach across the $k$-th decision region.

\emph{Second}, the sample $S_{x|k}$ has to be \emph{good} in the sense that 
$\fp = Q_k(M^{(k)})$ is sufficiently large (as noted earlier, see Fig.~\ref{fig:criticalmass}). 
This is somewhat related to the notion of \emph{representativeness} of a training set, 
\ie, a topic well studied in works on learnability.

Overall, given that $Q_k$ is $(b, c_{\beta})$-connected, 
mass concentration effects provably occur; yet, advantages 
only come to light under the conditions
outlined above. It is thus worth designing 
a \emph{regularization} strategy to encourage $(b, c_{\beta})$-connectedness 
during optimization. We describe such a strategy next.

\subsection{Regularization}
\label{subsection:regularization}

To encourage $(b, c_{\beta})$-connectedness of $Q_k$, it is obvious 
that we have to consider multiple training instances
of each class \emph{jointly}. To be practical, we integrate 
this requirement into the prevalent setting of learning with mini-batches. 

Our integration strategy is simple and, in fact, similar approaches (in a different context) have been
investigated in prior work \citep[see][]{Hoffer19a}. In detail, we construct each mini-batch, $\ttB$, as a collection of $\ttn$
sub-batches, \ie, $\ttB = (\ttB_1, \dots, \ttB_{\ttn})$.
Each sub-batch consists of $b$ samples from the \emph{same} class, thus the resulting mini-batch $\ttB$ is built from $\ttn \cdot b$ samples. 
Our regularizer is formulated as a loss term that penalizes deviations from a $\beta$-connected arrangement of the $z_i$ in each sub-batch $\ttB_j$. 
To realize this, we leverage a recent approach from \citet{Hofer19a} which introduces a \emph{differentiable} 
penalty on lifetimes of connected components (\wrt Vietoris-Rips persistent homology). 

Formally, let $\dagger(\ttB_i)$ contain the death-times (see Fig.~\ref{fig:vr}) computed for sub-batch $\ttB_i$. Then, given the hyper-parameter $\beta>0$, we set the \emph{connectivity penalty} for mini-batch $\ttB$ as
\begin{equation}
\label{eq:connectivity-loss}
  \mcL(\ttB) = \sum\limits_{i=1}^{\ttn} \sum\limits_{d \in \dagger(\ttB_i)} |d - \beta|
  \enspace.
\end{equation}
Notably, this is the same term as in \cite{Hofer19a}, 
however, motivated by a different objective.

Admittedly, to encourage $(b, c_{\beta})$-connectivity of $Q_k$, it would suffice to use a less restrictive variant and only penalize lifetimes \emph{greater} than $\beta$. However, we have empirically observed that Eq.~\eqref{eq:connectivity-loss} is more effective. 
This would imply that it is beneficial to prevent lifetimes from collapsing and, as a result, prevent $Q_k$ to become \emph{overly} dense. Currently, we can not formally explain this effect, but hypothesize that -- to some extent -- preventing lifetimes from collapsing preserves variance in the gradients, \ie, a property 
useful during SGD's \emph{drift} phase \cite{Shwartz17a}.

\section{Experiments}
\label{section:experiments}

For our experiments\footnote{\texttt{PyTorch} source code is available at  \url{https://github.com/c-hofer/topologically_densified_distributions}}, we draw on a setup common to many works 
in semi-supervised learning \cite{Laine17a,Oliver18a,Verma19a}, both in terms of dataset selection and network architecture. As small sample-size experiments are typically presented as \emph{baselines} in these works, we believe this to be an arguable choice. In particular, we present experiments on three (10 class) vision benchmark datasets:  MNIST, SVHN and CIFAR10. For MNIST and SVHN, we limit training 
data to 250 instances, on CIFAR10 to 500 (and 1,000), respectively.

\textbf{Architecture \& Optimization.} For CIFAR10 and SVHN
we use the CNN-13 architecture of \cite{Laine17a} which
already includes dropout regularization \cite{Srivastava14a}.
Only on MNIST we rely on a simpler CNN architecture with four convolutional blocks and max-pooling (w/o dropout). 
Both architectures have a final linear classifier 
$\gamma: \mathbb{R}^{128}\to \mathbb{R}^K$, use batch 
normalization \cite{Ioffe15a}, and 
fit our network decomposition of Eq.~\eqref{eqn:network}. 
Optimization is done by SGD with momentum (0.9) over 310 epochs with
cross-entropy loss and cosine learning rate annealing \cite{Loshchilov17a} (without restarts).
As all experiments use weight decay, it is important to note 
that batch normalization \emph{combined} with weight decay 
only exerts regularization on the classifier $\gamma$. 
In fact, several works have shown that the combination of
batch normalization and weight decay mainly affects the 
effective learning rate \cite{Laarhoven17a,Zhang19a}. 

The weighting of our regularization term is set such that the
range of the loss from Eq.~\eqref{eq:connectivity-loss} is
comparable, in range, to the cross-entropy loss. 
We choose a sub-batch size of $b=16$ and draw $\ttn=8$ 
sub-batches (see \S\ref{subsection:regularization}); 
this amounts to a total batch size of $128$.
While, empirically, this setting facilitates stable optimization 
(\wrt batch norm statistics), we acknowledge that further 
theoretical insights could lead to a more informed choice. 
Additional parameter details are provided when relevant 
(and full details can be found in the supplementary material).

First, in \S\ref{subsection:exp1}, we investigate to which 
extent the $(b,c_{\beta})$-connectivity property (imposed 
during optimization), translates to unseen data. 
In \S\ref{subsection:exp2}, we study 
the effect of $\beta$ and whether this parameter can be  
reliably cross-validated on a \emph{small} validation set. 
Finally, in \S\ref{subsection:exp3}, we compare to 
related work on regularizing statistics of 
internal representations.

\subsection{Evaluating $(b, c_{\beta})$-connectivity}
\label{subsection:exp1}

For a fixed $b$, we study how well $\beta$-connectivity
is achieved during optimization, as we vary $\beta$.
In accordance with Definition~\ref{defi:beta-connected}, we measure $\beta$-connectivity
via the lifetimes in the $0$-dimensional persistence barcodes, computed over 
500 random sub-batches (chosen from training/testing data).

Qualitatively, in Fig.~\ref{fig:beta_generalization} 
we see that increasing $\beta$ not only translates to an increase
of lifetimes in the internal representations of training instances, but 
equally translates to an increase in lifetimes on sub-batches of the  
testing data.
While we observe a slight offset in the lifetime average (training \emph{vs.} testing),
and an increase in variance, these effects are largely constant 
across $\beta$. This suggests the \emph{effect}
of regularization is qualitatively invariant to the choice
of $\beta$.

\begin{figure}[h!]
\centering{
\includegraphics[width=0.95\columnwidth]{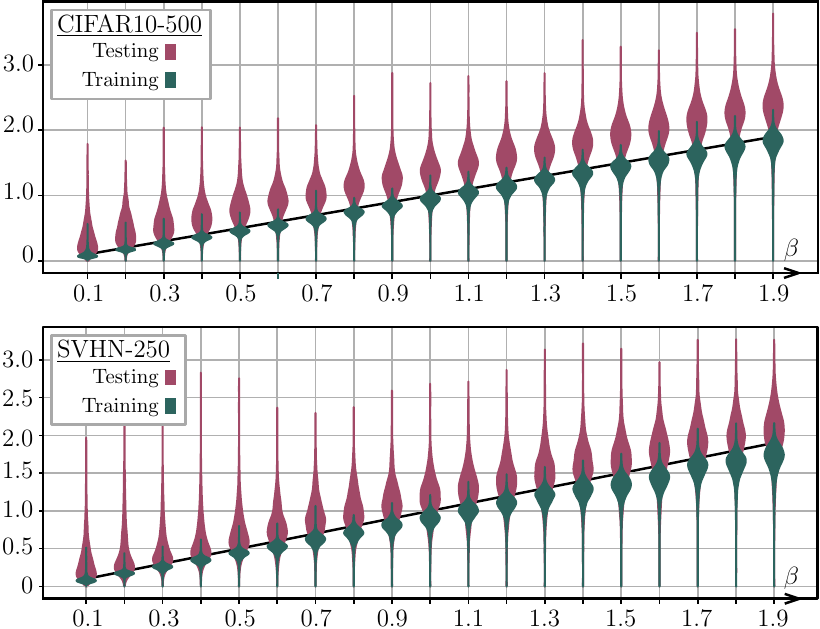}}
\vspace{-3pt}
\caption{Lifetime distribution computed over $500$ random sub-batches
(of size $16$) from (CIFAR10/SVHN) training and testing data, as a function of 
$\beta \in [0.1,1.9]$ (set during optimization).\label{fig:beta_generalization}}
\vspace{-8pt}
\end{figure}

\subsection{Selection of $\beta$}
\label{subsection:exp2}

Using Eq.~\eqref{eq:connectivity-loss} as a loss term 
requires to set $\beta$ a-priori. 
However, this choice can be crucial, as it assigns a notion of \emph{scale} to $\mathcal{Z}$ and thus  
interacts with the linear classifier. In particular, $\beta$
is interweaved with the Lipschitz constant of $\gamma$  
which is affected by weight decay. 

While, at 
the moment, we do not have sufficient theoretical insights 
into the interaction between weight decay on $\gamma$ and 
the choice of $\beta$, we argue that $\beta$ can still 
be cross-validated (a common practice for most
hyper-parameters). Yet, in small sample-size regimes, 
having a large labeled validation set is unrealistic. 
Thus, we study the behavior of cross-validating $\beta$, 
when the validation set is of size equal to the training
corpus. To this end, Fig.~\ref{fig:cv} shows the testing error
on CIFAR10 (using 500 training samples) and SVHN (using 
250 training samples), over a range of $\beta$. 
Additionally, we overlay the variation
in the error on the held-out validation sets. As we can see, 
the latter closely tracks the testing error as $\beta$
is increased from $0.1$ to $1.9$. This indicates that choosing
$\beta$ through cross-validation can be done effectively. 
Fig.~\ref{fig:cv} additionally reveals that the testing
error behaves smoothly around the optimal choice of $\beta$.
 
\begin{figure}[h!]
\centering{
\includegraphics[width=0.95\columnwidth]{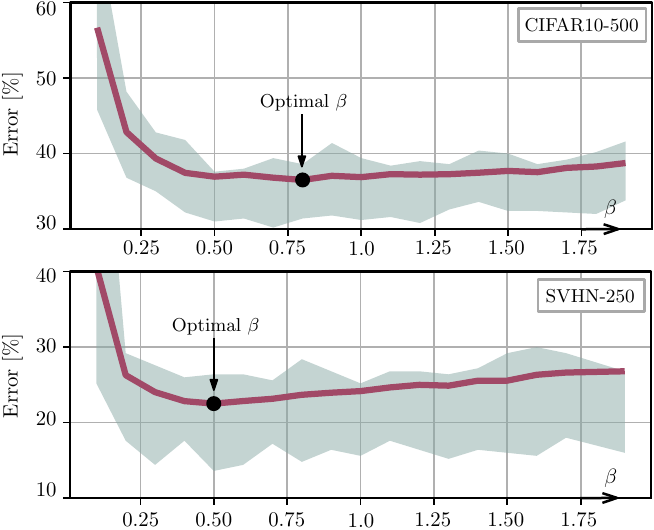}}
\caption{\label{fig:cv} 
Testing error (\textcolor{purple2}{purple}; averaged over 10 runs) 
over various choices of $\beta \in [0.1,1.9]$. 
The shaded region shows the variation in the 
testing error on small-validation sets. This indicates that
the choice of $\beta$ can be cross-validated effectively.}
\end{figure} 

\begin{table*}[t!]
\begin{center}
\begin{small}
\begin{tabular}{lcccc}
\toprule 
\textbf{Regularization}				     & 
{MNIST-250} & 
{SVHN-250} & 
{CIFAR10-500} & {CIFAR10-1k} \\
\midrule
%\multicolumn{5}{l}{\emph{\textcolor{purple3}{(below are the lowest achievable errors over a wide hyper-parameter grid)}}} \\
Vanilla 							 	& $7.1 \pm 1.0$ 	 & $30.1 \pm 2.9$ 	& $39.4 \pm 1.5$ 	& $29.5 \pm 0.8$ \\
\arrayrulecolor{black!50}
\midrule
~~+ Jac.-Reg. \cite{Hoffman19a}	    & $6.2 \pm 0.8$  & $33.1 \pm 2.8$ 	& $39.7 \pm 2.0$ 	& $29.8 \pm 1.2$ \\
~~+ DeCov 	  \cite{Cogswell16a}	    & $6.5 \pm 1.1$  & $28.9 \pm 2.2$ 	& $38.2 \pm 1.5$    & $29.0 \pm 0.6$ \\
~~+ VR        \cite{Choi19a}			& $6.1 \pm 0.5$  & $28.2 \pm 2.4$   & $38.6 \pm 1.4$ 	        & $29.3 \pm 0.7$ \\
~~+ cw-CR     \cite{Choi19a}		    & $7.0 \pm 0.6$  & $28.8 \pm 2.9$   & $39.0 \pm 1.9$   	& $29.1 \pm 0.7$ \\
~~+ cw-VR 	  \cite{Choi19a}		& $6.2 \pm 0.8$  & $28.4 \pm 2.5$   & $38.5 \pm 1.6$   	& $29.0 \pm 0.7$ \\
\midrule
~~+ Sub-batches 			 		& $7.1 \pm 0.5$	 & $27.5 \pm 2.6$ 	& $38.3 \pm 3.0$ 	& $28.9 \pm 0.4$ \\
~~+ Sub-batches + Top.-Reg. (\textbf{Ours})	 & 
	$\mathbf{5.6  \pm 0.7}$ & 
	$\mathbf{22.5 \pm 2.0}$ & 
	$\mathbf{36.5 \pm 1.2}$ & 
	$\mathbf{28.5 \pm 0.6}$ \\
\arrayrulecolor{black}
\midrule
%\multicolumn{5}{l}{\emph{\textcolor{purple3}{(below is the error when cross-validating $\beta$ with all other hyper-parameters fixed)}}} \\
~~+ Sub-batches + Top.-Reg. (\textbf{Ours}) $\ddagger$	 & 
	$\textcolor{purple3}{\mathbf{5.9  \pm 0.3}}$ & 
	$\textcolor{purple3}{\mathbf{23.3 \pm 1.1}}$ & 
	$\textcolor{purple3}{\mathbf{36.8 \pm 0.3}}$ & 
	$\textcolor{purple3}{\mathbf{28.8 \pm 0.3}}$ \\
\bottomrule
\end{tabular}
\end{small}
\end{center}
\caption{Comparison to state-of-the-art regularizers added to \emph{Vanilla} training which includes batch normalization, 
dropout ($0.5$; except for MNIST) and weight decay. Reported is the \underline{lowest achievable} test error [\%] ($\pm$ std. deviation) over a hyper-parameter grid, averaged over 10 cross-validation runs. Numbers attached to the dataset names indicates the number of training instances used. The last row ($\ddagger$) lists the results of our approach when $\beta$ is cross-validated (and all other hyper-parameters are fixed) as described in \S\ref{subsection:exp3}. \label{tbl:sota}}
\end{table*}

\subsection{Comparison to the state-of-the-art}
\label{subsection:exp3}

Finally, we present a comparison to different state-of-the-art 
regularizers. Specifically, we evaluate against works that 
regularize statistics of internal representations (right before
the linear classifier). This includes the \textit{DeCov} loss
of \citet{Cogswell16a}, as well as its class-wise extensions 
(\textit{cw-CR} and \textit{cw-VR}), proposed in \citet{Choi19a}.
As a representative of an alternative approach, 
we provide results when penalizing the network
Jacobian, as proposed in \cite{Sokolic17a, Hoffman19a}. For 
these comparison experiments, we empirically found a batch 
size of 32 to produce the best results. To account for the 
difference in the update steps of SGD \wrt to our approach 
(caused by the sub-batch construction), we adjusted the number 
of epochs accordingly. All approaches are evaluated
on the same training/testing splits and achieve
zero training error.

To establish a \emph{strong baseline}, we decided
to conduct an extensive hyper-parameter search over a grid 
of (1) initial learning rate, (2) weight decay and (3) 
weighting of the regularization terms. For each grid 
point, we run 10 cross-validation runs, average, and 
then pick the \emph{lowest achievable error} on the 
test set. This establishes an \emph{lower bound} on the 
error if hyper-parameters were chosen via a validation
set. Table \ref{tbl:sota} lists the corresponding results.

To test our regularizer against these lower bounds, we fix
all hyper-parameters and cross-validate $\beta$, as discussed 
in \S\ref{subsection:exp2}. Notably,  
topological regularization \emph{consistently} exhibits the 
lowest error, even when compared to the optimistic performance 
estimate of the other regularizers. This strongly supports
our claim that mass concentration is beneficial.

\section{Discussion}
\label{section:discussion}

As emphasized earlier, our theoretical results are \emph{relative} in nature, 
in particular, relative to a reference set $M$, naturally determined by
the representations of the training samples. 

In \S\ref{section:topologicallyregularizeddistributions}, we linked 
mass concentration to generalization and showed that mass in the $\beta$-extension 
of $M$ increases, as the probability $c_{\beta}$, \ie, the probability of a 
$b$-sized sample from $Q_k$ to be $\beta$-connected, 
increases. Results in Table~\ref{tbl:sota} empirically support this. 
\emph{However, can mass concentration be directly observed?} While it is challenging
to measure this, we can perform a proxy experiment. 
In detail, we select two models trained with equal $\beta$ 
and define the reference sets (per class) via balls of radius $r>0$, see Eq.~\eqref{eqn:trainingreferenceset},
around 500 randomly selected training samples.
By successively 
increasing $r$ and counting \emph{test samples} that occur in 
$B(z_i,r)$ and $B(z_i, r+\beta)$, resp., we obtain estimates of $\fp=Q_k(M)$ 
and $\fq=Q_k(M_{\beta})$. As $\beta$-connectivity is not \emph{strictly}
enforced, but used for regularization, we have to account for the lifetime 
shift seen in Fig.~\ref{fig:empiricalmass}. Hence, to estimate $\fp$ and $\fq$,  
we use $\beta=1.4$, which is higher than the sought-for $\beta=0.8$ during 
training. Fig.~\ref{fig:empiricalmass} (left) shows 
the mass estimates for two CIFAR10 models, trained on 500 and 1,000 
samples, respectively.
\begin{figure}[H]
\centering{
\includegraphics[width=\columnwidth]{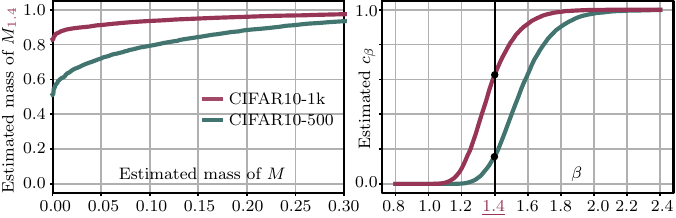}}
\caption{\label{fig:empiricalmass}
(\emph{Left}) Estimated mass in the reference set $M$ \emph{vs.} 
the estimated mass in $M_{\beta}$, shown for $\beta=1.4$; (\emph{Right}) Estimated probability
of a $b$-sized sample to be $\beta$-connected. All estimations are computed
on testing data.} 
\vspace{-3pt}
\end{figure}

As we can see -- especially for small (estimated) mass in the reference set --
the mass concentration effect is much stronger for the model trained
with 1,000 samples. The underlying reason is that using more samples
improves how $(b, c_{\beta})$-connectivity transfers to testing data.
Estimates for $c_{\beta}$ across both models confirm 
the latter, see Fig.~\ref{fig:empiricalmass} (right). This strongly indicates that mass concentration is not only
a theoretical result, but is indeed observed on real data. 
%Notably, 
%our use of balls to construct $M$ is not dictated by our theory, but 
%rather a matter of convenience. 

Overall, the presented analysis suggests that studying and \emph{controlling} 
topological properties of representations is promising. Yet, we have only 
started to scratch at the surface of how topological aspects influence 
generalization. We argue that further (formal) understanding of these connections 
could offer novel insights into the generalization puzzle.

\section*{Acknowledgements}
This work was partially funded by the Austrian Science Fund (FWF): project FWF P31799-N38 and the Land Salzburg (WISS 2025) under project numbers 
20102-F1901166-KZP and 20204-WISS/225/197-2019.

\clearpage
\bibliography{libraryChecked,booksLibrary}

\begin{thebibliography}{38}
\providecommand{\natexlab}[1]{#1}
\providecommand{\url}[1]{\texttt{#1}}
\expandafter\ifx\csname urlstyle\endcsname\relax
  \providecommand{\doi}[1]{doi: #1}\else
  \providecommand{\doi}{doi: \begingroup \urlstyle{rm}\Url}\fi

\bibitem[Arora et~al.(2018)Arora, Ge, Neyshabur, and Zhang]{Arora18a}
Arora, S., Ge, R., Neyshabur, B., and Zhang, Y.
\newblock Stronger generalization bounds for deep nets via a compression
  approach.
\newblock In \emph{ICML}, 2018.

\bibitem[Bartlett et~al.(2017)Bartlett, Foster, and Telgarsky]{Bartlett17a}
Bartlett, P., Foster, D., and Telgarsky, M.
\newblock Spectrally-normalized margin bounds for neural networks.
\newblock In \emph{NIPS}, 2017.

\bibitem[Behlarbi et~al.(2017)Behlarbi, Chatelain, Herault, and
  Adam]{Belharbi17a}
Behlarbi, S., Chatelain, C., Herault, R., and Adam, S.
\newblock Neural networks regularization through class-wise invariant
  representation learning.
\newblock \emph{arXiv}, 2017.
\newblock \url{https://arxiv.org/abs/1709.01867}.

\bibitem[Bianchini \& Scarselli(2014)Bianchini and Scarselli]{Bianchini14a}
Bianchini, M. and Scarselli, F.
\newblock On the complexity of neural network classifiers: A comparison between
  shallow and deep architectures.
\newblock \emph{IEEE Trans. Neural Netw. Learn. Syst.}, 25\penalty0
  (8):\penalty0 1533--1565, 2014.

\bibitem[Boissonnat et~al.(2018)Boissonnat, Chazal, and Yvinec]{Boissonnat18a}
Boissonnat, J.-D., Chazal, F., and Yvinec, M.
\newblock \emph{Geometric and Topological Inference}.
\newblock Cambridge Texts in Applied Mathematics. Cambridge University Press,
  2018.

\bibitem[Chen et~al.(2019)Chen, Ni, Bai, and Wang]{Chen19a}
Chen, C., Ni, X., Bai, Q., and Wang, Y.
\newblock A topological regularizer for classifiers via persistent homology.
\newblock In \emph{AISTATS}, 2019.

\bibitem[Choi \& Rhee(2019)Choi and Rhee]{Choi19a}
Choi, D. and Rhee, W.
\newblock Utilizing class information for deep network representation shaping.
\newblock In \emph{AAAI}, 2019.

\bibitem[Cogswell et~al.(2016)Cogswell, Ahmed, Girshick, Zitnick, and
  Batra]{Cogswell16a}
Cogswell, M., Ahmed, F., Girshick, R., Zitnick, L., and Batra, D.
\newblock Reducing overfitting in deep networks by decorrelating
  representations.
\newblock In \emph{ICLR}, 2016.

\bibitem[Cubuk et~al.(2019)Cubuk, Zoph, Man\'{e}, Vasudevan, and Le]{Cubuk19a}
Cubuk, E., Zoph, B., Man\'{e}, D., Vasudevan, V., and Le, Q.
\newblock Autoaugment: Learning augmentation strategies from data.
\newblock In \emph{CVPR}, 2019.

\bibitem[Dao et~al.(2019)Dao, Gu, Rather, Smith, Sa, and Re]{Dao19a}
Dao, T., Gu, A., Rather, A., Smith, V., Sa, C.~D., and Re, C.
\newblock A kernel theory of modern data augmentation.
\newblock In \emph{ICML}, 2019.

\bibitem[Edelsbrunner \& Harer(2010)Edelsbrunner and Harer]{Edelsbrunner2010}
Edelsbrunner, H. and Harer, J.~L.
\newblock \emph{Computational Topology : An Introduction}.
\newblock American Mathematical Society, 2010.

\bibitem[Glorot et~al.(2011)Glorot, Bordes, and Bengio]{Glorot11a}
Glorot, X., Bordes, A., and Bengio, Y.
\newblock Deep sparse rectifier neural networks.
\newblock In \emph{AISTATS}, 2011.

\bibitem[Golowich et~al.(2018)Golowich, Rakhlin, and Shamir]{Golowich18}
Golowich, N., Rakhlin, A., and Shamir, O.
\newblock Size-independent sample complexity of neural networks.
\newblock In \emph{COLT}, 2018.

\bibitem[Guss \& Salakhutdinov(2018)Guss and Salakhutdinov]{Guss18a}
Guss, W. and Salakhutdinov, R.
\newblock On characterizing the capacity of neural networks using algebraic
  topology.
\newblock \emph{arXiv}, 2018.
\newblock \url{https://arxiv.org/abs/1802.04443}.

\bibitem[Hofer et~al.(2019)Hofer, Kwitt, , Dixit, and Niethammer]{Hofer19a}
Hofer, C., Kwitt, R., , Dixit, M., and Niethammer, M.
\newblock Connectivity-optimized representation learning via persistent
  homology.
\newblock In \emph{ICML}, 2019.

\bibitem[Hoffer et~al.(2019)Hoffer, Ben-Nun, Hubara, Giladi, Hoefler, and
  Soudry]{Hoffer19a}
Hoffer, E., Ben-Nun, T., Hubara, I., Giladi, N., Hoefler, T., and Soudry, D.
\newblock Augment your batch: better training with larger batches.
\newblock \emph{arXiv}, 2019.
\newblock \url{https://arxiv.org/abs/1901.09335}.

\bibitem[Hoffman et~al.(2019)Hoffman, Roberts, and Yaida]{Hoffman19a}
Hoffman, J., Roberts, D., and Yaida, S.
\newblock Robust learning with jacobian regularization.
\newblock \emph{arXiv}, 2019.
\newblock \url{https://arxiv.org/abs/1908.02729}.

\bibitem[Ioffe \& Szegedy(2015)Ioffe and Szegedy]{Ioffe15a}
Ioffe, S. and Szegedy, C.
\newblock Batch normalization: Accelerating deep network training by reducing
  internal covariate shift.
\newblock In \emph{ICML}, 2015.

\bibitem[Joo et~al.(2020)Joo, Kang, and Kim]{Joo20a}
Joo, T., Kang, D., and Kim, B.
\newblock Regularizing activations in neural networks via distribution matching
  with the {W}asserstein metric.
\newblock In \emph{ICLR}, 2020.

\bibitem[Kolchinsky et~al.(2019)Kolchinsky, Tracey, and Kuyk]{Kolchinsky19a}
Kolchinsky, A., Tracey, B., and Kuyk, S.~V.
\newblock Caveats for information bottleneck in deterministic scenarios.
\newblock In \emph{ICLR}, 2019.

\bibitem[Laine \& Aila(2017)Laine and Aila]{Laine17a}
Laine, S. and Aila, T.
\newblock Temporal ensembling for semi-supervised learning.
\newblock In \emph{ICLR}, 2017.

\bibitem[Liao et~al.(2016)Liao, Schwing, Zemel, and Urtasun]{Liao16a}
Liao, R., Schwing, A., Zemel, R., and Urtasun, R.
\newblock Learning deep parsimonious representations.
\newblock In \emph{NIPS}, 2016.

\bibitem[Littwin \& Wolf(2018)Littwin and Wolf]{Littwin18a}
Littwin, E. and Wolf, L.
\newblock Regularizing by the variance of the activations’ sample-variances.
\newblock In \emph{NIPS}, 2018.

\bibitem[Loshchilov \& Hutter(2017)Loshchilov and Hutter]{Loshchilov17a}
Loshchilov, I. and Hutter, F.
\newblock {SGDR}: Stochastic gradient descent with warm restarts.
\newblock In \emph{ICLR}, 2017.

\bibitem[Nagarajan \& Kolter(2019)Nagarajan and Kolter]{Nagarajan19a}
Nagarajan, V. and Kolter, J.
\newblock Uniform convergence may be unable to explain generalization in deep
  learning.
\newblock In \emph{NeurIPS}, 2019.

\bibitem[Neyshabur et~al.(2017)Neyshabur, Bhojanapalli, McAllester, and
  Srebro]{Neyshabur17a}
Neyshabur, B., Bhojanapalli, S., McAllester, D., and Srebro, N.
\newblock Exploring generalization in deep learning.
\newblock In \emph{NIPS}, 2017.

\bibitem[Oliver et~al.(2018)Oliver, Odena, Raffel, Cubuk, and
  Goodfellow]{Oliver18a}
Oliver, A., Odena, A., Raffel, C., Cubuk, E., and Goodfellow, I.
\newblock Realistic evaluation of deep semi-supervised learning algorithms.
\newblock In \emph{NIPS}, 2018.

\bibitem[Rieck et~al.(2019)Rieck, Togninalli, Bock, Moor, Horn, Gumbsch, and
  Borgwardt]{Rieck19a}
Rieck, B., Togninalli, M., Bock, C., Moor, M., Horn, M., Gumbsch, T., and
  Borgwardt, K.
\newblock Neural persistence: A complexity measure for deep neural networks
  using algebraic topology.
\newblock In \emph{ICLR}, 2019.

\bibitem[Robins(2000)]{Robins00a}
Robins, V.
\newblock \emph{Computational topology at multiple resolutions: foundations and
  applications to fractals and dynamics}.
\newblock PhD thesis, University of Colorado, 6 2000.

\bibitem[Roux et~al.(2017)Roux, Manzagol, and Bengio]{Roux07a}
Roux, N., Manzagol, P.-A., and Bengio, Y.
\newblock Topmoumoute online natural gradient algorithm.
\newblock In \emph{NIPS}, 2017.

\bibitem[Shwartz-Ziv \& Tishby(2017)Shwartz-Ziv and Tishby]{Shwartz17a}
Shwartz-Ziv, R. and Tishby, N.
\newblock Opening the black box of deep neural networks via information.
\newblock \emph{arXiv}, 2017.
\newblock \url{https://arxiv.org/abs/1703.00810}.

\bibitem[Sokoli\'{c} et~al.(2017)Sokoli\'{c}, Giryes, Sapiro, and
  Rodrigues]{Sokolic17a}
Sokoli\'{c}, J., Giryes, R., Sapiro, G., and Rodrigues, M.
\newblock Robust large margin deep neural networks.
\newblock \emph{IEEE Trans. Signal Process.}, 65\penalty0 (16):\penalty0
  4265--4280, 2017.

\bibitem[Srivastava et~al.(2014)Srivastava, Hinton, Krizhevsky, Sutskever, and
  Salakhutdinov]{Srivastava14a}
Srivastava, N., Hinton, G., Krizhevsky, A., Sutskever, I., and Salakhutdinov,
  R.
\newblock Dropout: A simple way to prevent neural networks from overfitting.
\newblock \emph{JMLR}, 15:\penalty0 1929--1958, 2014.

\bibitem[van Laarhoven(2017)]{Laarhoven17a}
van Laarhoven, T.
\newblock $l_2$ regularization versus batch and weight normalization.
\newblock \emph{arXiv}, 2017.
\newblock \url{https://arxiv.org/abs/1706.05350}.

\bibitem[Verma et~al.(2019{\natexlab{a}})Verma, Lamb, Beckham, Najafi,
  Mitliagkas, Lopez-Paz, and Y.Bengio]{Verma19b}
Verma, V., Lamb, A., Beckham, C., Najafi, A., Mitliagkas, I., Lopez-Paz, D.,
  and Y.Bengio.
\newblock Manifold mixup: Better representations by interpolating hidden
  states.
\newblock In \emph{ICML}, 2019{\natexlab{a}}.

\bibitem[Verma et~al.(2019{\natexlab{b}})Verma, Lamb, Kannala, Bengio, and
  Lopez-Paz]{Verma19a}
Verma, V., Lamb, A., Kannala, J., Bengio, Y., and Lopez-Paz, D.
\newblock Interpolation consistency training for semi-supervised learning.
\newblock In \emph{IJCAI}, 2019{\natexlab{b}}.

\bibitem[Zhang et~al.(2017)Zhang, Bengio, Hardt, Recht, and
  Vinyals]{CZhang2017a}
Zhang, C., Bengio, S., Hardt, M., Recht, B., and Vinyals, O.
\newblock Understanding deep learning requires rethinking generalization.
\newblock In \emph{ICLR}, 2017.

\bibitem[Zhang et~al.(2019)Zhang, Wang, Xu, and Grosse]{Zhang19a}
Zhang, G., Wang, C., Xu, B., and Grosse, R.
\newblock Three mechanisms of weight decay regularization.
\newblock In \emph{ICLR}, 2019.

\end{thebibliography}
\bibliographystyle{icml2019}
\vfill

% 
% 
% 
%%%%%%%%%% Supp. Mat.  %%%%%%%%%%
\pagebreak
\onecolumn

\theoremstyle{definition}
\newtheorem{Sdefi}{Definition}

\theoremstyle{plain}
\newtheorem{Sthm}{Theorem}
\newtheorem{Slem}{Lemma}
\newtheorem{Scor}{Corollary}

\theoremstyle{remark}
\newtheorem{Srem}{Remark}

% Prefix to all labels ... 
\newcommand{\supprefix}{S}

\setcounter{section}{0}
\renewcommand{\thesection}{\supprefix\arabic{section}}

\setcounter{equation}{0}
\renewcommand{\theequation}{\supprefix\arabic{equation}}

\setcounter{figure}{0}
\renewcommand{\thefigure}{\supprefix\arabic{figure}}

\setcounter{table}{0}
\renewcommand{\thetable}{\supprefix\arabic{table}}

\renewcommand{\theSdefi}{\supprefix\arabic{Sdefi}}

\renewcommand{\theSlem}{\supprefix\arabic{Slem}}

\renewcommand{\theSthm}{\supprefix\arabic{Sthm}}

\renewcommand{\theScor}{\supprefix\arabic{Scor}}

\renewcommand{\theSrem}{\supprefix\arabic{Srem}}

\setcounter{page}{1}

\makeatletter

\renewcommand{\bibnumfmt}[1]{[S#1]}
\renewcommand{\citenumfont}[1]{S#1}

\begin{adjustwidth}{1.5cm}{1.5cm}

  \colorlet{restated}{purple3}
  % \linespread{1.15}
  % 
  \icmltitle{{\large {Supplementary Material}:} \\ \vspace{4pt} \papertitle}
  % 
  % 
  % 
  % !TEX root = ./topreg.tex

In this supplementary material, we provide (1) all proofs (in \S\ref{app:section:generalization} to \S\ref{app:section:monotonicity}) which were omitted in the main document, 
  as well as (2) full architectural details, hyper-parameters and optimization settings (in \S\ref{app:section:experimentaldetails}).
  Results which are restated from the main manuscript have the same numbering, while
  Definitions, Lemmas, etc., which are only present in the supplementary material have their labels suffixed by an ``S''.
  Additionally, restatements are given in \textcolor{restated}{purple}.
  
  \section{Generalization -- Proof of Lemma \ref{lem:generalization-in-feature-space}}
  \label{app:section:generalization}
  Recall that for
  $h: \mcX \rightarrow \mcY$ and $X \sim P$, and a given labeling function $c:\supp(P) \rightarrow \mcY$,  we define the 
  \emph{generalization error} as
  \[
    \expect\limits_{X \sim P}[\mathbb{1}_{h, c}(X)]\enspace.
  \]
  where
  \[
    \mathbb{1}_{h, c}(x) =
    \begin{cases}
      0, \quad h(x) = c(x), \\
      1, \quad \text{else}\enspace.
    \end{cases}
  \]
  We will now prove the following result of the main manuscript. 
  \textcolor{restated}{
    \lem@generalization@in@feature@space*
  }
  To prove this lemma, we first introduce an auxiliary indicator function in order to deal with possible label overlaps the mapping $\featext$ can impose in $\mcZ$. 
  \begin{Sdefi}
    \label{def:binary-indicator-feat}
    Let $h': \mcZ \rightarrow \mcY$ and $c^{\featext}(z) = c\big(\varphi^{-1}\big(\{z\}\big) \cap \supp(P) \big) \subseteq \mcY$.
    Then, we define
    \[
      \mathbb{1}^{\featext}_{h', c}(z) =
      \begin{cases}
        0, \quad |c^{\featext}(z)| = 1 \text{ and } h'(z) \in c^{\featext}(z), \\
        1, \quad \text{else}\enspace.
      \end{cases}
    \]
  \end{Sdefi}
  Setting $h'=\gamma$, this auxiliary indicator function, $\bb1_{\cls, c}^{\featext}$,  vanishes if and only if all $x \in \mcX$ which are mapped to an internal representation $z \in \mcZ$ \emph{have the same label $c(x)$}. 
  In other words, we \emph{pessimistically} assume that internal representation where this is not the case are falsely classified.
  Thus, the auxiliary indicator function composed with $\varphi$, $\mathbb{1}_{\gamma,c}^{\varphi} \circ \varphi$, has to be greater or equal than the original $\mathbb{1}_{\gamma \circ \varphi, c}$. 
  We formalize this insight next.   
  \begin{Slem}
    \label{lem:generalization-in-feature-space-aux-1}
    It holds that
    \[
      \mathbb{1}_{\cls \circ \featext, c}
      \leq
      \mathbb{1}^{\featext}_{\cls, c} \circ \featext
    \]
  \end{Slem}
  \begin{proof}
    Let $x \in \mcX$. It is sufficient to show that
%    \[
%      \mathbb{1}_{\cls \circ \featext, c}(x) = 1
%      \Rightarrow
%      \mathbb{1}^{\featext}_{\cls, c}\circ \featext(x) = 1
%      \enspace.
%    \]
%    This is equivalent to
    \[
      \mathbb{1}^{\featext}_{\cls, c}\circ \featext(x) = 0
      \Rightarrow
      \mathbb{1}_{\cls \circ \featext, c}(x) = 0
      \enspace.
    \]
    Thus, let $\mathbb{1}^{\featext}_{\cls, c}\circ \featext(x) = 0$.
    Then, by definition
    \begin{enumerate}[label=(\roman*)]
      \item $|c\big(\varphi^{-1}\big(\{\featext(x)\}\big)\big)| = 1$ and, \\
      \item  $\cls(\varphi(x)) \in c\big(\varphi^{-1}\big(\{\featext(x)\}\big)\big)$.
    \end{enumerate}
    By (i) there is some $y \in \mcY$ such that $c\big(\varphi^{-1}\big(\{\featext(x)\}\big)\big)= \{y\}$ and thus $c(x) = y$. 
    With this, and (ii), we get $\cls \circ \featext (x) = y$ and thus $\cls \circ \featext (x) = c(x)$.
    Therefore, $\mathbb{1}_{\cls \circ \featext, c}(x) = 0$ which concludes the proof.
  \end{proof}

  We now have all necessary tools to prove a slightly more general version of Lemma \ref{lem:generalization-in-feature-space} from the main manuscript. 
  \begin{Slem}
    \label{lem:generalization-in-feature-space-general}
    Let for any class $k \in \{1,\ldots,K\}$,  $C_k = \featext\left(c^{-1}\big(\{k\}\big)\right)$ be its internal representation and  $D_k = \cls^{-1}\big(\{k\}\big)$ its decision region in $\mcZ$.
    If
    $$\forall k: 1 - Q_k(D_k) \le \varepsilon_k\enspace,$$
    then
    \[
      \expect\limits_{X \sim P}[\mathbb{1}_{\cls\circ\featext, c}(X)]
      \leq
      \sum\limits_{k=1}^K
      \varepsilon_k
      \enspace.
    \]
  \end{Slem}

  \begin{proof}
    For brevity, let 
    %$M = \bigcup_{i=1}^K M_i$ 
    $\widehat{C}_k = c^{-1}\big(\{k\}\big)$ and 
    write $\bb1_{\gamma\circ\featext}$ instead of $\bb1_{\gamma\circ\featext, c}$.
    Then, we get
    \begin{align*}
      \expect\limits_{X \sim P}[\bb1_{\cls \circ \featext}(x)]
       & =
      \int\limits_{\mcX} \bb1_{\cls \circ \featext}(x) dP(x)      \\
       & =
      \int\limits_{\supp(P)} \bb1_{\cls \circ \featext}(x) dP(x)  \\
       & \leq
      \int\limits_{\supp(P)} \bb1^{\featext}_{\cls}\circ \featext (x) dP(x)
      \tag{by Lemma \ref{lem:generalization-in-feature-space-aux-1}} \\
       & =
      \int\limits_{\featext\left(\supp(P)\right)} \bb1^{\featext}_{\cls} (z) d Q(z)
      \tag{change of variables}                                    \\
       & =
      \sum\limits_{k=1}^K
      \int\limits_{\underbrace{\featext\left(\supp(P)\right)\cap D_k}_{ D_k^{\cap}}}
      \bb1^{\featext}_{\cls} (z) d Q(z)
      \tag{as $D_1,\ldots,D_K$ partition $\mcZ$}
      \enspace.
    \end{align*}

    For $1 \leq k \leq K$, let $D_k^{\cap} = \featext\left(\supp(P)\right)\cap D_k$ and consider each summation term separately.
    % First, as the $c^{-1}\big(\{k\}\big)$ partition $\supp(P)$ it holds that the $C_k = \featext\left(c^{-1}\big(\{k\}\big)\right)$ cover $\featext\big(\supp(P)\big)$.
    First, we can re-write $D_k^{\cap}$ as 
    $$
      D_k^{\cap} =
      \left(D_k^{\cap} \setminus \bigcup\limits_{\substack{i=1 \\ i \neq k}}^K C_i\right)
      \cup
      \left(D_k^{\cap} \cap \bigcup\limits_{\substack{i=1 \\ i \neq k}}^K C_i\right)
      \enspace.
    $$
    Second, consider
    \[
      z \in
      D_k^{\cap} \setminus \bigcup\limits_{\substack{i=1 \\ i \neq k}}^K C_i
      \enspace.
    \]
    Then, $z \in D_k$ and thus $\cls(z) = k$ (by definition of $D_k$).
    Further, we have $c\big(\varphi^{-1}\big(\{z\}\big) \cap \supp(P) \big) = \{k\}$ and thus,
    by Definition \ref{def:binary-indicator-feat},
    \begin{equation}
      \label{lem:generalization-in-feature-space:eq-1}
      \bb1^{\featext}_{\cls} (z) = 0
      \enspace.
    \end{equation}
    With this, considering each summation term yields
    \begin{align*}
      \int\limits_{ D_k^{\cap}}
      \bb1^{\featext}_{\cls} (z) d Q(z)
       & =
      \int\limits_{D_k^{\cap} \setminus \bigcup\limits_{\substack{i=1                  \\ i \neq k}}^K C_i}
      \bb1^{\featext}_{\cls} (z) d Q(z)
      +
      \int\limits_{D_k^{\cap} \cap \bigcup\limits_{\substack{i=1                       \\ i \neq k}}^K C_i}
      \bb1^{\featext}_{\cls} (z) d Q(z)                                                \\
       & =
      \int\limits_{D_k^{\cap} \cap \bigcup\limits_{\substack{i=1                       \\ i \neq k}}^K C_i}
      \bb1^{\featext}_{\cls} (z) d Q(z)
      \tag{by Eq.~\eqref{lem:generalization-in-feature-space:eq-1}}                    \\
       & =
      \int\limits_{\bigcup\limits_{\substack{i=1                                       \\ i \neq k}}^K (C_i \cap D_k^{\cap})}
      \bb1^{\featext}_{\cls} (z) d Q(z)                                                \\
       & \leq
      Q\left(\bigcup\limits_{\substack{i=1                                             \\ i \neq k}}^K (C_i \cap D_k^{\cap})\right)
      \tag{as $\bb1_{\gamma}^{\varphi} \leq 1$}\\
       & \leq
      Q\left(\bigcup\limits_{\substack{i=1                                             \\ i \neq k}}^K (C_i \cap D_k)\right)
      \tag{as $C_i \subset \varphi(\supp(P))$ }\\
       & \leq
      \sum\limits_{\substack{i=1                                             \\ i \neq k}}^K Q\left(C_i \cap D_k\right)
%       & \leq 
%      Q\left(\bigcup\limits_{\substack{i=1                                             \\ i \neq k}}^K (C_i \setminus D_i^{\cap})\right)
%      \tag{as $C_i \cap D_k^{\cap} \subseteq C_i \setminus D_i^{\cap}$ for $i \neq k$} \\
%       & \leq
%      \sum\limits_{\substack{i=1                                                       \\ i \neq k}}^K
%      Q\left(C_i \setminus D_i^{\cap}\right)                                           \\
%       & =
%      \sum\limits_{\substack{i=1                                                       \\ i \neq k}}^K
%      Q\left(C_i \setminus (D_i^{\cap} \cap C_i)\right)                                \\
%       & =
%      \sum\limits_{\substack{i=1                                                       \\ i \neq k}}^K
%      Q(C_i) - Q(D_i^{\cap} \cap C_i)                                                  \\
%      & = 
%      \sum\limits_{\substack{i=1                                                       \\ i \neq k}}^K
%      Q(C_i) - Q(D_i \cap C_i) 
%      \tag{as $C_i \subset \varphi(\supp(P))$ }
%      \enspace.
    \end{align*}
    In order to obtain the final result, we use the fact that the decision regions $D_k$ are disjoint and cover the representation space, i.e. $ \mcZ = \bigsqcup_{k=1}^K D_k$. Thus for any $1\le i\le K$,
  \begin{align}
  \label{lem:generalization-in-feature-space:eq-2}
  Q(C_i) = \sum\limits_{k=1}^K Q\left(C_i\cap D_k \right)
  \enspace.
  \end{align}
    Consequently, changing the summation order allows to simplify the bound from above to
    \begin{align*}
    \expect\limits_{X \sim P}[\bb1_{\cls \circ \featext}(x)]
     & \leq
     \sum\limits_{k=1}^K
     \sum\limits_{\substack{i=1                                             \\ i \neq k}}^K Q\left(C_i \cap D_k\right)
   %\\
      =
     \sum\limits_{i=1}^K
     \sum\limits_{\substack{k=1                                             \\ k \neq i}}^K Q\left(C_i \cap D_k\right) 
     \\
     & =
     \sum\limits_{i=1}^K \left( Q(C_i) - Q(C_i \cap D_i) \right)
     \tag{by Eq.~\eqref{lem:generalization-in-feature-space:eq-2}}
     \\
     &=
     \sum\limits_{i=1}^K Q(C_i) \left(1 - \frac{Q(C_i \cap D_i)} {Q(C_i)} \right)
     \\
     &=
     \sum\limits_{i=1}^K 
     Q(C_i) \left (1 - Q_i(D_i) \right)
     \tag{$*$}
     \\
     &\leq 
     \sum\limits_{i=1}^K Q(C_i) \varepsilon_i
     \tag{$**$}
     \\
     & \leq 
     \sum\limits_{i=1}^K \varepsilon_i \enspace,
    \end{align*}
    where ($*$) follows from the definition of the class-specific probability mass in Eq. \eqref{eqn:Q_k} and ($**$) holds by the assumption of the lemma.
  \end{proof}
  \begin{proof}[Proof of Lemma \ref{lem:generalization-in-feature-space}]
    By setting $\varepsilon_k = \varepsilon > 0 $ in Lemma \ref{lem:generalization-in-feature-space-general} we get the desired result.
  \end{proof}
  \section{A sufficient condition on \emph{not}-$\beta$-connectivity}
  \begin{Sdefi}
    Let $(\mcZ, \mathbb{d})$ be a metric space and $\emptyset \neq A, B \subset \mcZ$.
    We define the \emph{set margin} between $A$ and $B$ as
    \[
      \mathbb{m}(A, B) =
      \inf\limits_{a \in A, b \in B} \mathbb{d}(a, b)
      \enspace.
    \]
  \end{Sdefi}
  The following lemma formalizes the intuition that $z_1,\ldots,z_b$ cannot be $\beta$-connected if
  the $z_i$ are distributed among two sets which are separated by a sufficiently large set margin.
  \begin{Slem}
    \label{lem:sufficient-for-not-connectted}
    Let $(\mcZ, \mathbb{d})$ be a metric space and $\mbfz= (z_1,\ldots,z_b) \in \mcZ^b$.
    Define, for $l \in \N$ and $A, B \subset \mcZ$, the index sets
    \begin{align*}
      I_{A, \mbfz}  & = \{i \in [b]: z_i \in A\}, \quad \\
      I_{B, \mbfz}  & = \{i \in [b]: z_i \in B\}, \text{and} \quad \\
      I_{C, \mbfz}  & = \{i \in [b]: z_i \in (A\cup B)^{\complement}\}\enspace,
    \end{align*}
    where $[b] = \{1, \dots, b\}$ and $(A\cup B)^{\complement}$ denotes the set complement of $(A\cup B)$. 
    
    If
    \[
      \mathbb{m}(A, B) \geq l \cdot \beta
      \text{ and }
      |I_{A, \mbfz}|,~|I_{B, \mbfz}| \geq 1 \text{ and } |I_{C, \mbfz}| \leq l - 1
    \]
    then
    \[
      \indc_b^{\beta}(\mbfz) = 0\enspace,
    \]
    \ie, $\mbfz$ is not $\beta$-connected.
  \end{Slem}

  \begin{proof}
    We prove this by way of contradiction. 
    For brevity, let $\indc_b^{\beta} = \indc$.
    Consider $\mbfz = (z_1,\ldots,z_b)$ as above and assume $\indc(\mbfz) = 1$.
    Let, w.l.o.g., $z_1 \in A$ and $z_b \in B$.
    Then, there is a path of distinct nodes
    \[
      z_1 \leftrightarrow \dots \leftrightarrow z_b
    \]
    connecting $z_1$ and $z_b$ with line segments of length $< \beta$.
    However, by assumption $\mathbb{m}(A,B) \geq \beta$, and thus there is a sub-path
    \[
      z_{i_1} \leftrightarrow z_{i_2} \leftrightarrow \dots \leftrightarrow z_{i_m}
    \]
    such that $z_{i_1} \in A$, $z_{i_m} \in B$ and $z_{i_2}, \dots, z_{i_{m-1}} \in (A \cup B)^{\complement}$.
    Thus, we get
    \begin{align*}
      \mathbb{d}(z_{i_1}, z_{i_m})  \leq
      \mathbb{d}(z_{i_1}, z_{i_2})
      + \dots +
      \mathbb{d}(z_{i_{m-1}}, z_{i_m})
      <
      (m-1) \cdot \beta\enspace.
    \end{align*}
    By construction, $z_{i_2},\ldots,z_{i_{m-1}} \in C$ and thus $\{i_2,\ldots,i_{m-1}\} \subseteq I_{C,\mbfz}$. Hence, $m - 2 \leq |I_{C,\mbfz}| \leq l - 1$. 
    Therefore, $(m-1)\cdot \beta \leq l\cdot \beta$, leading to
    \[
      \mathbb{d}(z_{i_1}, z_{i_m})
      <
      l\cdot \beta\enspace.
    \]
    This directly contradicts $\mathbb{m}(A, B) \geq l \cdot \beta$.
  \end{proof}
  \section{Concentration results}

  % \begin{Sdefi}
  %   \label{def:not-connected-index-set}
  %   Let $b, l \in \N$ then we define
  %   \begin{align*}
  %     I(b, l) =
  %     \left\{\right.
  %     (n_1, n_2, n_3) \in \{0, \dots, b\}^3:
  %     n_1 + n_2 + n_3 = b, 1 \leq n_1, 1 \leq n_2, n_3 \leq l - 1
  %     \left.\right\}
  %     \enspace.
  %   \end{align*}
  % \end{Sdefi}
  % 
  \begin{Slem}
    \label{lem:conectivity-impact-core}
    Let $(\mcZ, \mathbb{d})$ be a metric space and $\Sigma$ the corresponding Borel $\sigma$-algebra.
    Further, let $b \in \N$, $\beta > 0$ and $Q$ be a $(b,c_{\beta})$-connected probability measure (cf. Definition~\ref{defi:bbetac}) 
    on $(\mcZ, \Sigma)$.
  For $l \in \N$ and $A, B \in \Sigma$, such that
    $\mathbb{m}(A, B) \geq l \cdot \beta$, the following inequality holds
    \begin{align*}
      1-c_{\beta}  \geq
      \sum\limits_{\substack{(n_1, n_2, n_3) \in I(b, l)}}
      \frac{b!}{n_1! n_2! n_3!}
      \cdot Q(A)^{n_1}Q(B)^{n_2}(1 - Q(A) - Q(B))^{n_3}\enspace,
    \end{align*}
    where the index set $I(b,l)$ is given as
    \begin{align}
      \label{lem:connectivity-impact-core:eq:index-set}
      I(b, l) =
      \left\{\right.
      (n_1, n_2, n_3) \in \{0, \dots, b\}^3:
      n_1 + n_2 + n_3 = b,\ 1 \leq n_1,\ 1 \leq n_2,\ n_3 \leq l - 1
      \left.\right\}
      \enspace.
    \end{align}
  \end{Slem}
  \begin{proof}
    Let $C = (A \cup B)^{\complement}$.
    The proof is structured in \emph{three parts}: First, we construct an auxiliary random 
  variable that captures the scattering of $b$-sized samples across $A$, $B$ and $C$. 
  This allows us to express probabilities for different scattering configurations.
  Second, we describe scattering configurations where $\beta$-connectivity cannot 
  be satisfied. Finally, by combining both previous parts, we derive the claimed 
  inequality.
  
    {\bf Part I.}
    First, define the categorical function 
    \begin{equation*}
      f: \mcZ \to \{1,2,3\}, \quad f(z) =
      \begin{cases}
        1, \quad z \in A, \\
        2, \quad z \in B, \\
        3, \quad z \in C = (A \cup B)^{\complement}
        \enspace.
      \end{cases}
    \end{equation*}
    For a random variable $Z \sim Q$, we now consider the random variable $f \circ Z$ with
    \begin{align*}
      Q(\{f \circ Z = 1\}) & = Q(A) \\
      Q(\{f \circ Z = 2\}) & = Q(B) \\
      Q(\{f \circ Z = 3\}) & = Q(C)
      \enspace.
    \end{align*}
    By drawing $b$-times i.i.d. from $f \circ Z$ and counting the occurrences of $1,2,3$, we get a multinomially distributed random variable, $K$. 
    This means that, for    
    \begin{equation*}
      \{K = (n_1, n_2, n_3)\}
    \end{equation*}
    where $n_1 + n_2 + n_3 = b$, 
    it holds that
    \begin{equation*}
      %\bbP\left(\{\mcK = (n_1, n_2, n_3)\}\right)=
      Q^b\left(\{K = (n_1, n_2, n_3)\}\right) =
      \frac{b!}{n_1! n_2! n_3!}
      \cdot Q(A)^{n_1}Q(B)^{n_2}(1 - Q(A) - Q(B))^{n_3}
      \enspace.
    \end{equation*}

    {\bf Part II}. Similar to Lemma \ref{lem:sufficient-for-not-connectted}, we define 
    \begin{align*}
      I_{A, \mbfz}  & = \{i \in I: z_i \in A\}, \\
      I_{B, \mbfz}  & = \{i \in I: z_i \in B\}, \text{and}\\
      I_{C, \mbfz}  & = \{i \in I: z_i \in (A\cup B)^{\complement}\}\enspace.
    \end{align*}
    Then, by construction, it holds that
    \begin{equation*}
      \big\{K = (n_1, n_2, n_3)\big\}
      =
      \left\{
      \mbfz \in \mcZ^b:
      \big(| I_{A, \mbfz} |, | I_{B, \mbfz} |, | I_{C, \mbfz} |\big) = (n_1, n_2, n_3)
      \right\}
      \enspace.
    \end{equation*}
    Now let $(n_1, n_2, n_3) \in I(b, l)$ and consider $\mbfz \in \{K= (n_1, n_2, n_3)\}$.
    By definition of $I(b,l)$, we get
    \begin{equation*}
      1 \leq |I_{A, \mbfz}|  ,\quad
      1 \leq |I_{B, \mbfz} |, \quad \text{and }\quad
      |I_{C, \mbfz}|  \leq l - 1
      \enspace.
    \end{equation*}
    Remember that, by assumption, $\mathbb{m}(A, B) \geq l\cdot\beta$ and thus, by Lemma \ref{lem:sufficient-for-not-connectted},
    $\mbfz \in \{\indc =0 \}$, \ie,
    the points $z_1,\ldots,z_b$ are not $\beta$-connected.
    This yields the following implication:
    \begin{equation*}
      (n_1,n_2,n_3) \in I(b,l) \Rightarrow \{K= (n_1, n_2, n_3)\} \subseteq \{\indc = 0\}\enspace.
    \end{equation*}

    {\bf Part III}. Combining the results of Part I and II, we obtain the following inequality:
    \begin{align*}
      1 - c_{\beta} & = Q^b(\{\indc = 0\})      \\
                    & \geq
      Q^b
      \left(
      \bigcup\limits_{\substack{(n_1, n_2, n_3) \\ \in I(b, l)}}
      \{K = (n_1, n_2, n_3)\}
      \right)                                   \\
                    & =
      \sum\limits_{\substack{(n_1, n_2, n_3)    \\   \in I(b, l)}}
      Q^b\big(\{K = (n_1, n_2, n_3)\}\big)   \\
                    & =
      \sum\limits_{\substack{(n_1, n_2, n_3)    \\ \in I(b, l)}}
      \frac{b!}{n_1! n_2! n_3!}
      \cdot Q(A)^{n_1}Q(B)^{n_2}(1 - Q(A) - Q(B))^{n_3}
      \enspace.
    \end{align*}
  \end{proof}
  With Lemma~\ref{lem:conectivity-impact-core} in mind, we can restate the definition 
  of the polynomial $\Psi$ from the main manuscript.
  \textcolor{restated}{\restatable@defi@psifunction*}

  While all previous results in this supplementary material are stated for a 
  probability measure $Q$ on $\mathcal{Z}$, the results equally transfer
  to $Q_k$, \ie, the restriction of probability measure $Q$ to a particular class $k$, 
  which is the specific setting considered in the main part of the manuscript. 
  \textcolor{restated}{\restatable@thm@psifunction*}
  \begin{proof}
    The proof relies on Lemma \ref{lem:conectivity-impact-core} with
    \[
      A = M\text{ and  } B = (M_{l \cdot \beta})^{\complement}\enspace.
    \]
    For $Q_k$ as $Q$, we get 
    \begin{align*}
    1 - c_{\beta} 
    & \geq 
    \sum\limits_{\substack{(n_1, n_2, n_3)    \\ \in I(b, l)}}
      \frac{b!}{n_1! n_2! n_3!}
      \cdot Q(A)^{n_1}Q(B)^{n_2}(1 - Q(A) - Q(B))^{n_3}\\
      &=
      \sum\limits_{\substack{(n_1, n_2, n_3)    \\ \in I(b, l)}}
      \frac{b!}{n_1! n_2! n_3!}
      \cdot Q(M)^{n_1}\big(1 - Q(M_{l \cdot \beta})\big)^{n_2}\big(1 - Q(M) - (1 - Q(M_{l \cdot \beta}))\big)^{n_3}\\
      &=
      \sum\limits_{\substack{(n_1, n_2, n_3)    \\ \in I(b, l)}}
      \frac{b!}{n_1! n_2! n_3!}
      \cdot Q(M)^{n_1}\big(1 - Q(M_{l \cdot \beta})\big)^{n_2}\big(Q(M_{l \cdot \beta})-Q(M)\big)^{n_3}\\
      &=
      \sum\limits_{\substack{(u, v, w)    \\ \in I(b, l)}}
      \frac{b!}{u! v! w!}
      \cdot p^{u}(1 - q)^{v}(q-p)^{w}\enspace,
    \end{align*}
    where, in the last equality, we have set $p = Q(M)$, $q=Q(M_{l \cdot \beta})$ and 
    renamed the indices. 
  \end{proof}
  \clearpage
  \section{Properties of $\Psi$}
  In the main manuscript, we list three important properties of $\Psi$. These are:
  \begin{enumerate}[label=(\arabic*)]
    \item $\Psi$ is \emph{monotonically increasing} in $p$, 
    \item $\Psi$ is \emph{monotonically decreasing} in $q$, and
    \item $\Psi$ is \emph{monotonically increasing} in $l$ and $\Psi$ vanishes for $q=1$.
  \end{enumerate}
  While the latter trivially follows from Definition~\ref{def:psi-function}, 
  as $(1-q) = (1-1) = 0$ is always present with non-zero exponent, properties (1) and (2) need more careful (tedious) consideration.
  The monotonicity in $l$ results from the fact that increasing $l$ increases the size of $I(b,l)$ and thus more non-negative
  terms are present in the summation in $\Psi$ (see Definition~\ref{def:psi-function}) .
  We start by providing two beneficial ways of re-writing the index set, $I(\cdot, \cdot)$, which is used to define $\Psi$. 
  \begin{Slem}
    \label{lem:psi-function-alternative-indices}
    Let $b, l \in \N$ and define
    \[
      g(x) = \max \{1, b -x -l +1\}\enspace.
    \]
    This yields the following re-write of the index set as follows:
    \begin{equation}
      \begin{gathered}
        I(b,l) = \left\{
        (n_1, n_2, n_3) \in \{0, \dots, b\}^3:
        n_1 + n_2 + n_3 = b, 1 \leq n_1, 1 \leq n_2, n_3 \leq l - 1
        \right\}\\
        =\\
        \bigcup\limits_{n_1=1}^{b-1}
        \left\{
        (n_1, n_2, b - n_1 - n_2): g(n_1) \leq n_2 \leq b- n_1
        \right\}\\
        =\\
        \bigcup\limits_{n_2=1}^{b-1}
        \left\{
        (n_1, n_2, b - n_1 - n_2): g(n_2) \leq n_1 \leq b- n_2
        \right\}\\
      \end{gathered}
    \end{equation}
     \end{Slem}
  \begin{proof}
    We only show the first equality as the second is analogous with switched roles of $n_1$ and $n_2$.

    \textbf{Part I} ($\boldsymbol{\subseteq}$):
    Let $(k, i, j) \in I(b,l)$, \ie, $(k,i,j) \in \{0, \dots, b\}^3$ such that
    \begin{align*}
      \text{(1)   }&\quad k + i + j = b,\quad \\
      \text{(2)   }&\quad 1\leq k, \quad \\
      \text{(3)   }&\quad 1 \leq i\,\text{, and }\quad \\
      \text{(4)   }&\quad j \leq l-1 \enspace.
    \end{align*}
    From (1) it follows that $j = b - k - i$. \\
    From (4) we get
    \[
      b - k - i \stackrel{\text{1)}}{=} j \stackrel{\text{(4)}}{\leq} l - 1
      \Leftrightarrow
      b -k -l + 1 \leq i\enspace.
    \]
    Combining this with (3), we conclude
    \[
      g(k) = \max\{1, b -k -l+1\} \leq i
      \enspace.
    \]
    From (1), we see that $i \leq b-k$, as $i+j = b-k$ and $j>0$. This means 
    \begin{equation}
      (k, i, j) \in \left\{
      (k, n_2, b - k - n_2): g(k) \leq n_2 \leq b-k
      \right\}
      \enspace.
    \end{equation}
    Finally, (1), (2), and (3) yield $1 \leq k \leq b-1$ and therefore
    \begin{equation}
      (k, i, j)
      \in
      \bigcup\limits_{n_1=1}^{b-1}
      \left\{
      (n_1, n_2, b - n_1 - n_2): g(n_1) \leq n_2 \leq b-k
      \right\}
    \end{equation}
    which concludes the ``$\subseteq$'' part.

    \textbf{Part II} ($\boldsymbol{\supseteq}$):
    Let $1 \leq n_1 \leq b-1$ and consider
    \begin{equation}
      (k, i, j)
      \in
      \left\{
      (n_1, n_2, b - n_1 - n_2): g(n_1) \leq n_2 \leq b-k
      \right\}
      \enspace.
    \end{equation}
    Then, $k + i + j = b$ and $1 \leq k, 1 \leq i$. \\
    For the last condition, \ie, $j \leq l - 1$, consider
    \[
      j = b -k - i \leq b - k - g(k)\enspace.
    \]
    We next distinguish the two possible outcomes of $g(k)$:

    {\underline{Case 1}: $g(k) = b -k -l + 1$}: Then, we get
    \begin{align*}
    j & \leq b - k - g(k) \\
      & = b -k - (b -k -l +1)\\
      & = l - 1\enspace.
    \end{align*}
    %Then we get $j \leq b -k -(b -k -l +1) = l - 1$

    {\underline{Case 2}: $g(k) = 1$}: Then, by definition of $g(k)$, we get
    \begin{align*}
             b -k -l +1 & \leq 1 \\
    \Leftrightarrow b -k & \leq l \\
    \Leftrightarrow b-k-1 & \leq l-1
    \end{align*}
    %\[
    %  b -k -l +1 \leq 1 \Leftrightarrow b -k \leq l \Leftrightarrow b-k-1 \leq l-1
    %\]
    and therefore
    \begin{align*}
    j & \leq b - k - g(k) \\
      & = b-k-1\\
      & \leq l - 1\enspace.
    \end{align*}
  \end{proof}
  We now use the results of Lemma~\ref{lem:psi-function-alternative-indices} to re-arrange the sum in the definition of $\Psi$.
  \begin{Scor}
    \label{cor:psi-indexing}
    For $g(x) = \max\{1, b-x-l+1\}$, it holds that
    \begin{equation}
      \label{cor:psi-indexing-eq:1}
      \Psi(p, q \parasep b, l) =
      \sum\limits_{n_1 = 1}^{b-1}
      \sum\limits_{n_2 = g(n_1)}^{b-n_1}
      % \binom{b}{n_1, n_2, b-n_1 - n_2}
      \frac{b!}{n_1! n_2! (b-n_1-n_2)!}
      p^{n_1}
      (1-q)^{n_2}
      (q-p)^{b-n_1 - n_2}
    \end{equation}
    and
    \begin{equation}
      \label{cor:psi-indexing-eq:2}
      \Psi(p, q \parasep b, l) =
      \sum\limits_{n_2 = 1}^{b-1}
      \sum\limits_{n_1 = g(n_2)}^{b-n_2}
      % \binom{b}{n_1, n_2, b-n_1 - n_2}
      \frac{b!}{n_1! n_2! (b-n_1-n_2)!}
      p^{n_1}
      (1-q)^{n_2}
      (q-p)^{b-n_1 - n_2}\enspace.
    \end{equation}
  \end{Scor}
  In the following lemma we will prove the claimed monotonicity properties of $\Psi$ by (i) using the previously derived rearrangements of the summation and 
  (ii) considering the corresponding derivatives. 
  \begin{Slem}
    \label{lem:psi-function-monotony}
    Let $b, l \in \N$ and $p_0, q_0 \in (0, 1)$ arbitrary but fixed.
    Then, it holds that
    \begin{equation*}
      \label{lem:psi-function-monotony-eq:p}
      \text{\it (1) }
      \Psi(\cdot, q_0 \parasep b, l) \text{ is monotonically increasing on } [0, q_0]
    \end{equation*}
    and
    \begin{equation*}
      % \label{lem:psi-function-monotony-eq:q}
      ~~~~~\text{\it (2) }
      \Psi(p_0, \cdot \parasep b, l) \text{ is monotonically decreasing on } [p_0, 1]
      \enspace.
    \end{equation*}
  \end{Slem}
%Alternative approach
\begin{proof}
  {\it Ad (1).}
  For brevity, we write $q$ instead of $q_0$.
  First, we leverage Corollary \ref{cor:psi-indexing}, Eq. \eqref{cor:psi-indexing-eq:2},
  and re-arrange the sum: 
  \begin{align*}
  \Psi(p, q \parasep b, l) & =
  \sum\limits_{n_2 = 1}^{b-1}
  \sum\limits_{n_1 = g(n_2)}^{b-n_2}
  % \binom{b}{n_1, n_2, b-n_1 - n_2}
  \frac{b!}{n_1! n_2! (b-n_1-n_2)!}
  p^{n_1}
  (1-q)^{n_2}
  (q-p)^{b-n_1 - n_2}                  \\
  & =
  \sum\limits_{n_2 = 1}^{b-1}
  (1-q)^{n_2}
  \underbrace{
    \sum\limits_{n_1 = g(n_2)}^{b-n_2}
    % \binom{b}{n_1, n_2, b-n_1 - n_2}
    \frac{b!}{n_1! n_2! (b-n_1-n_2)!}
    p^{n_1}
    (q-p)^{b-n_1 - n_2}}_{=A_{n_2}(p)}
  \enspace.
  \end{align*}
  For studying the monotonicity properties of $\Psi(\cdot, q \parasep b, l)$, it is sufficient to consider $A_{n_2}$, for $1\leq n_2 \leq b- 1$.
  
  We define two auxiliary functions
  \begin{align*}
  a_{n_1}(p) &= \frac{b!}{n_1! n_2! (b-n_1-n_2)!}
  p^{n_1}
  (q-p)^{b-n_1 - n_2},\\
  c_{n_1}(p) &= \frac{b!}{n_1! n_2! (b-n_1-n_2)!}
  p^{n_1}
  (b-n_1 - n_2)(q-p)^{b-n_1 - n_2 - 1}\enspace.
  \end{align*}
  Note that $A_{n_2}(p) = \sum_{n_1=g(n_2)}^{b-n_2} a_{n_1}(p)$  and  that
  $$
  c_{n_1}(p) = 
  \begin{cases} 
  \frac{b!}{n_1! n_2! (b-n_1-n_2-1)!}
  p^{n_1}
  (q-p)^{b-n_1 - n_2 - 1} \ge 0, & 0 \le n_1 < b - n_2\\
  0, & n_1 = b - n_2  
  \end{cases},
  $$
  because by assumption $0\le p \le q \le 1$.
  As we will show below, 
  $$\frac{\partial a_{n_1}(p)}{\partial p} = c_{n_1 - 1}(p) - c_{n_1}(p).$$
  Hence, $$\frac{\partial A_{n_2}(p)}{\partial p} = \sum\limits_{n_1 = g(n_2)}^{b-n_2} (c_{n_1 - 1}(p) - c_{n_1}(p)) = c_{g(n_2) - 1}(p) - c_{b-n_2}(p) = c_{g(n_2) - 1}(p)\ge0.$$
  Consequently, $A_{n_2}$  is monotonically increasing and thus, so is $\Psi(\cdot, q \parasep b, l)$.
  
  It remains to calculate the derivative $\frac{\partial a_{n_1}(p)}{\partial p}$:
  \begin{align*}
  \frac{\partial a_{n_1}(p)}{\partial p} &
  =  \frac{b!}{n_1! n_2! (b-n_1 - n_2)!}
  \cdot
  \Big[
  n_1 p^{n_1-1}
  (q-p)^{b-n_1 - n_2}
  -
  p^{n_1}
  (b-n_1 - n_2)(q-p)^{b-n_1 - n_2 - 1}
  \Big] \\
  &= 
  \frac{b!}{n_1! n_2! (b-n_1 - n_2)!}
  \cdot 
  n_1 p^{n_1-1}
  (q-p)^{b-n_1 - n_2}\\
  & \qquad -
  \frac{b!}{n_1! n_2! (b-n_1 - n_2)!}
  \cdot
  p^{n_1}
  (b-n_1 - n_2)(q-p)^{b-n_1 - n_2 - 1}  \\
  & = 
    \frac{b!}{(n_1-1)! n_2! (b-n_1 - n_2)!}
  \cdot 
  p^{n_1-1}
  (q-p)^{b-n_1 - n_2}\\
  & \qquad -
  \frac{b!}{n_1! n_2! (b-n_1 - n_2)!}
  \cdot
  p^{n_1}
  (b-n_1 - n_2)(q-p)^{b-n_1 - n_2 - 1}\\
  & = c_{n_1-1}(p) - c_{n_1}(p)
  \end{align*}
  
  {\it Ad 2).}
  The proof is rather similar to the first part.
  Nevertheless, we will exercise it for completeness.
  %elaborate it in all its fatiguing tediousness for the interested reader's convenience.\\
  For brevity, we will write $p$ instead of $p_0$.
  Again we start by leveraging Corollary \ref{cor:psi-indexing}, but, this time, use Eq. \eqref{cor:psi-indexing-eq:1}, and re-arrange the sum,
  \begin{align*}
  \Psi(p, q \parasep b, l) & =
  \sum\limits_{n_1 = 1}^{b-1}
  \sum\limits_{n_2 = g(n_1)}^{b-n_1}
  % \binom{b}{n_1, n_2, b-n_1 - n_2}
  \frac{b}{n_1! n_2! (b-n_1 - n_2)!}
  p^{n_1}
  (1-q)^{n_2}
  (q-p)^{b-n_1 - n_2}          \\
  & =
  \sum\limits_{n_1 = 1}^{b-1}
  p^{n_1}
  \underbrace{
    \sum\limits_{n_2 = g(n_1)}^{b-n_1}
    % \binom{b}{n_1, n_2, b-n_1 - n_2}
    \frac{b}{n_1! n_2! (b-n_1 - n_2)!}
    (1-q)^{n_2}
    (q-p)^{b-n_1 - n_2}
  }_{A_{n_1}(q)}
  \end{align*}
  For studying the monotonicity properties of $\Psi(p, \cdot \parasep b, l)$ it is sufficient to consider $A_{n_1}$, for $1\leq n_1 \leq b- 1$.
    
  Again, we define two auxiliary functions
  \begin{align*}
  a_{n_2}(q) &= \frac{b!}{n_1! n_2! (b-n_1-n_2)!}
  (1-q)^{n_2}
  (q-p)^{b-n_1 - n_2},\\
  c_{n_2}(q) &= \frac{b!}{n_1! n_2! (b-n_1-n_2)!}
  (1-q)^{n_2}
  (b-n_1 - n_2)(q-p)^{b-n_1 - n_2 - 1}.
  \end{align*}
  Note that $A_{n_1}(q) = \sum_{n_2=g(n_1)}^{b-n_1} a_{n_2}(q)$  and  that
  $$
  c_{n_1}(q) = 
  \begin{cases}   
  \frac{b!}{n_1! n_2! (b-n_1-n_2-1)!}
  (1-q)^{n_2}
  (q-p)^{b-n_1 - n_2 - 1} \ge 0, & 0\le n_2 < b - n_1 \\
  0, & n_2 = b - n_1\\ 
  \end{cases},
  $$
  because by assumption $0\le p \le q \le 1$.
  
  As we will show below, 
  $$\frac{\partial a_{n_2}(q)}{\partial q} = c_{n_2}(q) - c_{n_2 - 1}(q).$$
  Hence, $$\frac{\partial A_{n_1}(q)}{\partial q} = \sum\limits_{n_2 = g(n_1)}^{b - n_1} (c_{n_2}(q) - c_{n_2-1}(q)) = c_{b-n_1}(q) - c_{g(n_1) - 1}(q) = - c_{g(n_1) - 1}(q)\le0.$$
  Consequently, $A_{n_1}$  is monotonically decreasing and thus, so is $\Psi(p,\cdot, \parasep b, l)$.
  
  It remains to calculate the derivative $\frac{\partial a_{n_2}(q)}{\partial q}$:
  \begin{align*}
  \frac{\partial a_{n_2}}{\partial q}(q)
  & = 
  - \frac{b!}{n_1! n_2! (b-n_1 - n_2)!}
  n_2 (1-q)^{n_2 - 1}
  (q-p)^{b-n_1 - n_2}
    \\
  & \qquad  +
  \frac{b!}{n_1! n_2! (b-n_1 - n_2)!} 
  (1-q)^{n_2}
  (b-n_1 - n_2)(q-p)^{b-n_1 - n_2 - 1} \\
  & =
  - \frac{b!}{n_1! (n_2-1)! (b-n_1 - n_2)!}
  (1-q)^{n_2 - 1}
  (q-p)^{b-n_1 - n_2}
  \\
  & \qquad  +
  \frac{b!}{n_1! n_2! (b-n_1 - n_2)!} 
  (1-q)^{n_2}
  (b-n_1 - n_2)(q-p)^{b-n_1 - n_2 - 1} \\
  & =
  - c_{n_2-1} + c_{n_2}
  \end{align*}
\end{proof}

  \section{Monotonicity properties of $\mcR$}
  \label{app:section:monotonicity}
  \begin{Sdefi}
    \begin{align*}
      \mcR_{b, c_{\beta}}(p, l)
      = \min_q\Big\{
      q \in [p,1]:
      1 - c_{\beta}
      \geq
      \Psi(p, q \parasep b, l)
      \Big\}
    \end{align*}
  \end{Sdefi}
  \begin{Slem}
    $\mcR_{b, c_{\beta}}(p, l)$ is 
    \begin{align*}
      & \text{(1) monotonically increasing in } p, \text{and} \\
      & \text{(2) monotonically increasing in } l
      \enspace.
    \end{align*}
  \end{Slem}
  \begin{proof}
    {\it Ad (1).} Let $p, \hat{p} \in [0, 1]$ with $p < \hat{p}$.\\
    First, assume $\mcR_{b, c_{\beta}}(p, l) \in [p, \hat{p}]$.
    Then, by definition,
    $\mcR_{b, c_{\beta}}(\hat{p}, l) \in [\hat{p}, 1]$ and thus
    \[
      \mcR_{b, c_{\beta}}(p, l) \leq \mcR_{b, c_{\beta}}(\hat{p}, l)\enspace.
    \]

    Second, assume $\mcR_{b, c_{\beta}}(p, l) \notin [p, \hat{p}]$, \ie,
    $\mcR_{b, c_{\beta}}(p, l) \in [\hat{p},1]$.
    Let
    \[
      A = \{q \in [\hat{p},1]: 1 - c_{\beta} \geq \Psi(p, q \parasep b, l)\}
    \]
    and
    \[
      ~~~B = \{q \in [\hat{p},1]: 1 - c_{\beta} \geq \Psi(\hat{p}, q \parasep b, l)\}
      \enspace.
    \]
    By Lemma \ref{lem:psi-function-monotony}, $\Psi$ is monotonically increasing in
    $p$ and thus
    \[
      \Psi(\hat{p}, q \parasep b, l) \geq \Psi(p, q \parasep b, l)
      \text{ for }
      q \in [\hat{p}, 1]
      \enspace.
    \]
    This implies $B \subseteq A$ and therefore
    \begin{align*}
      \mcR_{b, c_{\beta}}(p, l) = \min A \leq \min B = \mcR_{b, c_{\beta}}(\hat{p}, l)
      \enspace.
    \end{align*}

    {\it Ad 2).}
    Let $l < \hat{l}$.
    It follows from the definition of $I$, see Definition \ref{def:psi-function}, that
    $I(b, l) \subseteq I(b, \hat{l})$.
    As all addends in the sum defining $\Psi$ are positive, the claim follows.
  \end{proof}
  
  \section{Experimental details}
  \label{app:section:experimentaldetails}
  For reproducibility, we provide full architectural details, optimization settings and 
  hyper-parameters.
  
  \subsection{Architecture}
  On SVHN and CIFAR10, we use the CNN-13 architecture of \citep[][Table 5]{Laine17a}, 
  without the Gaussian noise input layer. The configuration is provided in Table~\ref{tbl:cnn13}
  below (essentially reproduced from the original paper). \texttt{BN} denotes 2D batch normalization 
  \cite{Ioffe15a}, \texttt{LReLU} denotes leaky ReLU activation with $\alpha=0.1$.
  
  \begin{table}[H]
  \begin{small}
  \centering{
  \begin{tabular}{llcc}
  \toprule
  & & \texttt{BN} & \texttt{LReLU} \\
  \midrule
  \textbf{Input}    & $32 \times 32$ RGB image & & \\
  \midrule
  \texttt{Conv (2D)}  & \texttt{Filters: 128; Kernel: 3x3; Pad: 1} & \checkmark & \checkmark \\
  \texttt{Conv (2D)}  & \texttt{Filters: 128; Kernel: 3x3; Pad: 1} & \checkmark & \checkmark \\
  \texttt{Conv (2D)}  & \texttt{Filters: 128; Kernel: 3x3; Pad: 1} & \checkmark & \checkmark \\
  \texttt{MaxPool (2D)} & \texttt{Window: 2x2, Stride: 2, Pad: 0}    & - & -\\
  \texttt{Dropout (0.5)}& & & \\
  \texttt{Conv (2D)}  & \texttt{Filters: 256; Kernel: 3x3; Pad: 1} & \checkmark & \checkmark \\
  \texttt{Conv (2D)}  & \texttt{Filters: 256; Kernel: 3x3; Pad: 1} & \checkmark & \checkmark \\
  \texttt{Conv (2D)}  & \texttt{Filters: 256; Kernel: 3x3; Pad: 1} & \checkmark & \checkmark \\
  \texttt{MaxPool (2D)} & \texttt{Window: 2x2, Stride: 2, Pad: 0}    & - & -\\
  \texttt{Dropout (0.5)}& & & \\
  \texttt{Conv (2D)}  & \texttt{Filters: 512; Kernel: 3x3; Pad: 0} & \checkmark & \checkmark \\
  \texttt{Conv (2D)}  & \texttt{Filters: 256; Kernel: 1x1; Pad: 0} & \checkmark & \checkmark \\
  \texttt{Conv (2D)}  & \texttt{Filters: 128; Kernel: 1x1; Pad: 0} & \checkmark & \checkmark \\
  \texttt{AvgPool (2D)} & \texttt{Window: 6x6, Stride: 2, Pad: 0}    & - & -\\
  \midrule
  \multicolumn{4}{c}{\emph{Vectorize to} $z \in \mathbb{R}^{128}$ (this is where the 
  topological regularizer operates)}\\
  \midrule
  \texttt{FullyConn.} & \texttt{128 -> 10} & - & -  \\
  \bottomrule
  \end{tabular}}
  \caption{\label{tbl:cnn13} CNN-13 architecture. The network part up to the vectorization 
  operation constitutes $\varphi$.}
  \end{small}
  \end{table}
  
  On MNIST, we use a simpler CNN architecture, listed in Table~\ref{table:mnistnet}.  
  
  \begin{table}[H]
  \centering{
    \begin{small}
  \begin{tabular}{llcc}
  \toprule
  & & \texttt{BN} & \texttt{LReLU} \\
  \midrule
  \textbf{Input}    & $28 \times 28$ grayscale image & & \\
  \midrule
  \texttt{Conv (2D)}  & \texttt{Filters: 8; Kernel: 3x3; Pad: 1}   & \checkmark & \checkmark \\
  \texttt{MaxPool (2D)} & \texttt{Window: 2x2, Stride: 2, Pad: 0}    & - & -\\
  \texttt{Conv (2D)}  & \texttt{Filters: 32; Kernel: 3x3; Pad: 1}    & \checkmark & \checkmark \\
  \texttt{MaxPool (2D)} & \texttt{Window: 2x2, Stride: 2, Pad: 0}    & - & -\\
  \texttt{Conv (2D)}  & \texttt{Filters: 64; Kernel: 3x3; Pad: 1}    & \checkmark & \checkmark \\
  \texttt{MaxPool (2D)} & \texttt{Window: 2x2, Stride: 2, Pad: 0}    & - & -\\
  \texttt{Conv (2D)}  & \texttt{Filters: 128; Kernel: 3x3; Pad: 1}   & \checkmark & \checkmark \\
  \texttt{MaxPool (2D)} & \texttt{Window: 2x2, Stride: 2, Pad: 0}    & - & -\\
  \midrule
  \multicolumn{4}{c}{\emph{Vectorize to} $z \in \mathbb{R}^{128}$ (this is where the 
  topological regularizer operates)}\\
  \midrule
  \texttt{FullyConn.} & \texttt{128 -> 10} & - & -  \\
  \bottomrule
  \end{tabular}
    \end{small}}
  \caption{\label{table:mnistnet} MNIST CNN architecture.}
  \end{table}
  
  \subsection{Optimization \& Augmentation}
  
  For optimization, we use SGD with momentum (set to 0.9). As customary in the literature \citep[see e.g.,][]{Verma19a}, training images
  for CIFAR10 and SVHN are augmented by (1) zero-padding images by 2 pixel on each side, followed by 
  random cropping of a $32 \times 32$ region and (2) random horizontal flipping (with probability 0.5).  
  On MNIST, no augmentation is applied. All images are further normalized by subtracting the mean and 
  dividing by the standard deviation. Importantly, these statistics are computed for each cross-validation
  split separately (from the training instances), as this is the only practical choice in a small sample-size regime. 
 
  \subsection{Hyper-parameter settings}
  
  Except for the last row, Table~\ref{tbl:sota} lists the \emph{best achievable error} for different regularization approaches 
  over a hyper-parameter grid to establish a \emph{lower bound} on the obtainable error.
  
  Across all experiments, weight decay on $\varphi$ is fixed to $1\mathrm{e}{-3}$. Due to the use of batch normalization, this 
  primarily affects the effective learning rate \citep[see][]{Laarhoven17a,Zhang19a}. On MNIST, we fix the initial learning 
  rate to 0.1. On SVHN and CIFAR10, we additionally experimented with an initial learning rate of 
  $0.3$ and $0.5$ and include these in 
  our hyper-parameter grid. The learning rate is annealed following the cosine learning rate annealing proposed
  in \cite{Loshchilov17a}.
  
  \textbf{Regularization}. The hyper-parameter grid is constructed as follows: \emph{weight decay} on $\gamma$ is varied in  
  $\{1\mathrm{e}{-3}, 5\mathrm{e}{-4},1\mathrm{e}{-4}  \}$. For \emph{Jacobian regularization} \cite{Hoffman19a},  the weighting of the regularization term is varied in $\{1\mathrm{e}{-3}, 0.05, 0.01, 0.1\}$. For \emph{DeCov} \cite{Cogswell16a}, \emph{VR} and \emph{cw-CR/VR} 
  \cite{Choi19a}, weighting of the regularization term is varied in $\{1\mathrm{e}{-4}, 1\mathrm{e}{-3}, 0.01, 0.1\}$.
  All these different choices are evaluated over 10 cross-validation runs with exactly the same training/testing split configuration.
  
  \textbf{Topological regularization}. To evaluate the sub-batch construction in combination with our proposed
  topological regularizer, the initial learning rate on MNIST is fixed to 0.1 and 0.5 for SVHN and CIFAR10. With
  topological regularization enabled, this always produced stable results. Importantly, weight decay for $\gamma$ is fixed to to $0.001$, except for the CIFAR10-1k experiment, where we set it to 
  $5\mathrm{e}{-4}$. \emph{The lowest achievable error is thus only selected by varying $\beta$} in $[0.1,1.9]$. 
  
  To obtain 
  the last row of Table~\ref{tbl:sota}, we no longer sweep over $\beta$, but select $\beta$ via cross-validation 
  over held-out validation sets of size 250 on SVHN and MNIST, and 500/1,000 on CIFAR10, respectively.
  
  \vskip2ex
  \textcolor{purple3}{\emph{The full, \texttt{PyTorch}-compatible, source code will be made publicly available at \url{https://github.com/c-hofer/topologically_densified_distributions}.}}

\end{adjustwidth}
\end{document}